\documentclass[10pt]{article}

\usepackage{mystyle}
\RequirePackage[top=2.25cm, bottom=2.5cm, left=2.5cm, right=2.5cm, columnsep=0.65cm]{geometry}
\newcommand{\han}[1]{{\textcolor{cyan}{[Han: #1]}}}

\newcommand{\pre}{\mathrm{pre}}
\newcommand{\sou}{\mathrm{s}}

\title{\LARGE The Power of Active Multi-Task Learning in Reinforcement Learning from Human Feedback}
\author{Ruitao Chen \thanks{Peking University. Email: \texttt{chenruitao@stu.pku.edu.cn}}\qquad Liwei Wang \thanks{Peking University. Email: \texttt{wanglw@pku.edu.cn}}}
\date{}
\begin{document}
\maketitle

\begin{abstract}
    Reinforcement learning from human feedback (RLHF) has contributed to performance improvements in large language models. To tackle its reliance on substantial amounts of human-labeled data, a successful approach is multi-task representation learning, which involves learning a high-quality, low-dimensional representation from a wide range of source tasks. In this paper, we formulate RLHF as the contextual dueling bandit problem and assume a common linear representation. We demonstrate that the sample complexity of source tasks in multi-task RLHF can be reduced by considering task relevance and allocating different sample sizes to source tasks with varying task relevance. We further propose an algorithm to estimate task relevance by a small number of additional data and then learn a policy. We prove that to achieve $\varepsilon-$optimal, the sample complexity of the source tasks can be significantly reduced compared to uniform sampling. Additionally, the sample complexity of the target task is only linear in the dimension of the latent space, thanks to representation learning.
\end{abstract}

\section{Introduction}

Recent advancements in large language models (LLMs) have significantly improved performance across various tasks. Many breakthroughs in large language models, including the notable language model application ChatGPT, are based on reinforcement learning from human feedback (RLHF) \citep{christiano2017deep,stiennon2020learning,wu2021recursively,ouyang2022training,chen2022human,bai2022training,xiong2023iterative,zhu2023principled,zhong2024dpo,rafailov2024direct}, a method that learns from human preference. The process of RLHF is to learn a reward
from data in the form of pairwise or K-wise comparisons between responses. The data typically consists of multiple potential responses to prompts from human users, along with ranks of all responses according to user preferences. A reward model is then trained based on a maximum likelihood estimator (MLE) using the data.

A challenge with this approach is its reliance on substantial amounts of expensive human-labeled data. Typically, we have a wide range of tasks, but only a small amount of data for each task. We can take advantage of multi-task representation learning \citep{caruana1997multitask} to address this issue. To be specific, with data labeled by different criteria of different users (different but similar tasks), we can learn a common low-dimensional representation. Thus, given only a limited
amount of data for a target task, we can achieve high performance with the aid of the high-quality low-dimensional representation learned from a wide range of data of different but related source tasks. This idea is successful in natural language processing \citep{ando2005framework,devlin2018bert,radford2019language,liu2019multi,brown2020language}.

Previous works typically only
consider uniform sampling from each source task. However, as pointed out in \citet{chen2022active}, the total sample complexity of source tasks can be further reduced if we consider the priority and relevance of source tasks and allocate different sample sizes for different source tasks. Intuitively, we should include more task data from high-quality sources or task data more relevant to the target tasks. In this paper, we aim to formally propose an algorithm to figure out the importance (task relevance) of different source tasks and show how to allocate sample size to reduce the total sample complexity of source tasks in multi-task RLHF. 

In this paper, we consider two settings. We formulate RLHF as the contextual dueling bandit problem and assume a common linear representation. We also assume that the user preference in the target task is a linear mixture of that in the source tasks. If we know the task relevance of source tasks, we can allocate a sample size proportional to the task relevance for each source task. We further consider the case that we do not know the task relevance. We only need a small number of additional sample complexity of the source tasks and the target task to estimate the task relevance. Then we can use a similar approach as in the first case. We prove that in both cases, to find an $\varepsilon-$optimal policy, the sample complexity of the source tasks
can be largely reduced compared with uniform sampling and the sample complexity of the target task is only linear in the dimension of the latent space, with the aid of representation learning. We remark that our results can be easily extended to active multi-task learning in logistic regression or logistic bandits settings.

\subsection{Related Works}

\paragraph{Contexual Dueling Bandits.}
 In the context of the Dueling Bandit setting, one learns the reward model by getting feedback from pairwise or K-wise comparisons \citep{yue2009interactively,yue2011beat,yue2012k,ailon2014reducing,zoghi2014relative,komiyama2015regret,gajane2015relative,saha2019active,saha2019pac,saha2022efficient,ghoshal2022exploiting,zhu2023principled}. In the pairwise comparison setting, it is similar to the logistic bandits setting which is one Generalized Linear Bandits of
crucial practical interest. The Generalized Linear Bandits setting was first studied by \citet{filippi2010parametric}. Then several algorithms and regret guarantees were proposed \citep{russo2013eluder,russo2014learning,li2017provably,abeille2017linear,jun2017scalable,dong2018information,dumitrascu2018pg}. Particularly for logistic bandits, the bound normally depends on a factor $\kappa$ which characterizes the degree of non-linearity of the logistic function. \citet{faury2020improved} improved the dependency and finally \citet{abeille2021instance} proved the optimal dependency. 

\paragraph{Multi-task Learning.} Learning a common low-dimensional representation among related tasks to facilitate learning is known as multi-task representation learning. Representation learning has been highly successful in various applications \citep{bengio2013representation}, for instance in reinforcement learning \citep{wilson2007multi,teh2017distral,d2024sharing}. Theoretical results for multi-task learning date back to \citet{baxter2000model}. Many theoretical results showed the benefit of representation learning in a variety of settings \citep{maurer2016benefit,du2020few,tripuraneni2021provable}, especially in reinforcement learning \citep{brunskill2013sample,arora2020provable,hu2021near,bose2024offline,d2024sharing,ishfaq2024offlinemultitaskrepresentationlearning}. \citet{chen2022active} and \citet{wang2023improved} focused on active multi-task learning. The idea is to learn the task relevance of source tasks to the target task and set uneven sample sizes for each source task, in order to further reduce the sample complexity of source tasks. 

\subsection{Notations}

For any positive integer $M$, we denote the set $\{1, 2, \cdots, M\}$ by $[M]$.  We use $\mathcal O$ and $\lesssim$ to hide absolute constants, and $\Tilde{\mathcal O}$ to further hide logarithmic factors. For any matrix $\mathbf{A}$, its operator norm and Frobenius norm are denoted by $\|\mathbf{A}\|$ and $\|\mathbf{A}\|_F$, respectively. For any positive definite matrix $\mathbf A$, denote the Mahalanobis norm
with regards to $\mathbf A$ by $\|x\|_\mathbf{A}=\sqrt{x^\top\mathbf{A}x}$. The $\psi_2-$norm of a random variable $x$ is defined as \begin{align*}
\|x\|_{\psi_2} = \inf\{t>0:\mathbb E\exp(x^2/t^2)\leq 2\}, 
\end{align*}
and that of a $d-$dimensional random vector $X$ is defined as 
\begin{align*}
\|X\|_{\psi_2} = \sup_{x\in \mathbb{S}^{d-1}}\|\langle X,x \rangle\|_{\psi_2}, \quad \mathbb{S}^{d-1}=\{x\in\mathbb R^d: \|x\|_2=1\}. 
\end{align*}

\section{Preliminaries}

Reinforcement learning from human feedback \citep{christiano2017deep,ouyang2022training} is often formulated as the contextual dueling bandit problem \citep{yue2012k,saha2021optimal}. A contextual bandit can be described by a tuple $(\mathcal S, \cA, r^*)$, where $\mathcal S$ is the context set, $\cA$ is the action set, and $r^*: \mathcal S \times \cA \rightarrow [0, 1]$ is the reward function. Given an context $s \in \mathcal S$ and two actions $(a^1, a^2)\in \cA\times \cA$, the environment generate the preference feedback $y \in \{0,1\}$ according to the Bradley-Terry model \citep{bradley1952rank}:
\begin{align} \label{eq:BT:model}
\PP(y = 1) = \frac{\exp(r^*(s, a^1))}{\exp(r^*(s, a^1)) + \exp(r^*(s, a^2))} =  \mu\big( r^*(s, a^1) - r^*(s, a^2) \big),
\end{align}
where the link function 
\begin{align*}
\mu(t) = 1/(1+\exp(-t))
\end{align*}
is the sigmoid function. In the application of RLHF in LLMs, each context $s$ in $\mathcal{S}$ represents a prompt, while $(a^1, a^2)$ denotes two potential responses, and  $y$ signifies the preference provided by a human evaluator.

We focus on the multi-task RLHF setting, where each task is a contextual dueling bandit problem. Specifically, we have $M$ source tasks $\{1, 2, \dots, M\}$ and one target task $M+1$, where each task $\cT_m$ is a contextual during bandit $(\mathcal S_m, \cA_m, r_m^*)$. 

%\han{then introduce specific assumptions for multitask RLHF; (1) linear reward; (2) shared representations; Then write the objective of multi-task learning and data assumptions. Please see e.g., \citet{chen2022active}, \citet{hu2021near} and other references for the description.}

Specifically, we assume that the reward is linear in features of the context and action pair. That is, for all $m\in[M+1]$, there exists a known common feature map $\phi: \mathcal S\times \cA \rightarrow \mathbb R^d$ where $\mathcal S, \cA$ are unions of all context set and action set, respectively. We also assume that there exists a $d$-dimensional parameter $\theta^*_m$ such that $r^*(s_m,a_m) = \phi(s_m, a_m)^\top \theta^*_m$ for all $(s_m,a_m)\in \mathcal S_m\times \cA_m$. For notation convenience, we use $x_m = \phi(s_m, a^1_m)-\phi(s_m, a^2_m)$ to denote the different of features of the context and action pair. Then, given $(s_m, a^1_m, a^2_m)$, \eqref{eq:BT:model} specifies the distribution of the preference feedback $y_m$
\begin{align*}
\mathbb P(y_m=1) = \mu(x_m^\top \theta^*_m), \quad m\in[M+1]. 
\end{align*}
Throughout this paper, we may drop the corner mark $M+1$ for terms related to the target task for simplicity, for instance, notation $r^* = r_{M+1}^*$ and $\theta^* = \theta_{M+1}^*$. 
We impose the following assumption for the parameters across tasks, which is standard in the literature of multi-task learning \citep{du2020few,tripuraneni2021provable,chen2022active,wang2023improved}.

\begin{assumption}
\label{smp:lowdim}
There exists a linear feature extractor $\mathbf B\in\mathbb R^{d\times k}$ and $k$-dimensional
coefficients $\{w_m^*\}_{m=1}^{M+1}$ such that $\theta^*_m=\mathbf B w_m^*$ for all $m\in[M+1]$, where $k \le d$. In addition, there exists $\nu = \{\nu_m\}_{m=1}^M\in\mathbb R^M$ such that the linear parameter $\theta^*$ of target task can be represented as 
\begin{align*}
\theta^*=\sum_{m=1}^M \nu_m \theta^*_m.
\end{align*}
\end{assumption}

Assumption~\ref{smp:lowdim} posits two essential conditions for the parameter space: (i)
all parameters $\{\theta^*_m\}_{m\in [M+1]}$ lie in a common $k$ dimensional subspace. This, when $k \ll d$, harnesses the potential for greater sample efficiency in multi-task learning due to the utilization of a common representational framework across various tasks \citep{du2020few}. And (ii) the true parameter $\theta^*$ of the target task can be expressed as a linear combination of the true parameters from all source tasks. The coefficients $\{\nu_m\}_{m \in [M]}$ quantifies the degree of relevance between the target task and the source tasks.

In addition to the aforementioned assumption, we also need a set of standard regularity conditions for logistic bandits or linear contextual dueling bandits \citep{faury2020improved,abeille2021instance,ji2024reinforcement}. For ease of presentation, for all $m \in [M+1]$, we define the domain $\Theta_m$ encompasses linear coefficients $\theta_m$ associated with task $m$, ensuring that the true coefficient $\theta_m^* \in \Theta_m$.
\begin{assumption}
\label{smp:bound}
All linear contextual dueling bandits satisfy the following conditions: 
\begin{itemize} 
\item There exists an absolute constant constant $L>0$, in all $M+1$ tasks, for all data features $x_m$ in task $m$, and all $\theta_m\in \Theta_m$, we have $\theta_m^\top x_m\leq L$.
\item There exists bounds $B_x, B_\theta>0$, in all $M+1$ tasks, for all data features $x_m$ in task $m$, we have $\|x_m\|_2\leq B_x$, and the true parameter $\|\theta^*_m\|_2\leq B_\theta$. And we have $B_x^2B_\theta^2\lesssim kd$. 
\end{itemize}
\end{assumption}

We define a function $\kappa$ as follows: 
\begin{align}
\label{eq:kappa}
\kappa(L) = \sup_{|t|\leq |L|}(1/\mu'(t)) = \mathcal O(\exp(|L|)), \quad L\in\mathbb R.
\end{align}
The sample complexity normally depends on $\kappa$, which scales exponentially. 

\paragraph{Active Multi-task Learning.} We have access to sample from a fix distribution $\mu_m$ of $x_m$ and get a random feedback $y_m$ corresponding to $x_m$ sampled from $\mu_m$, for all tasks $m\in[M+1]$. To reduce the total sample size, we shall set uneven sample numbers for $M$ source tasks according to the task relevance $\nu$ in Assumption \ref{smp:lowdim}. The main idea in Active Multi-task Learning is the two steps. First estimate the task relevance $\nu$ by using a small number of data from the source tasks and the target task, and then set uneven sample size for $M$ source tasks according to the preceding estimation to reduce the total sample complexity. 

We set the sample complexity of source task $m$ to be $n_m$ and collect $n_m$ i.i.d. samples expressed as $\{(s_{m, i}, a_{m, i}^1, a_{m, i}^2)\}_{i = 1}^{n_m}$, while the corresponding features are $\{x_{m,i}\}_{i = 1}^{n_m}$, from source task $m$. By construction, $\{x_{m,i}\}_{i = 1}^{n_m}$ are generated i.i.d from distribution $\mu_m$. Using these samples, we can learn the feature extractor $B$ and the subspace spanned by $\theta^*_m$ equivalently. Finally, the sample size needed to estimate $\theta^*$ of the target task is reduced with the aid of the estimation of the subspace. We
mainly concerned about the total sample number of the source tasks.

For $n_m$ i.i.d. samples $\{(s_{m, i}, a_{m, i}^1, a_{m, i}^2)\}_{i = 1}^{n_m}$, while the corresponding features are $\{x_{m,i}\}_{i = 1}^{n_m}$, from source task $m$. We rewrite data to matrix form, $\mathbf X_m\in \mathbb R^{n_m\times d}$ with rows representing $x_{m,1}^\top, \dots, x_{m,n_m}^\top$ and $y_m\in\mathbb R^{n_m}$ representing corresponding feedback term. 

\paragraph{Fisher Information Assumption.} Following the RLHF pipeline \citep{ouyang2022training}, we estimate the reward function parameter $\theta^*_m$ via Maximum Likelihood Estimation (MLE). For all tasks $m\in[M+1]$, the log-likelihood for data features $x_m$ and feedback $y_m$ is represented as 
\begin{align}
\ell\left(x_m^\top\theta^*_m, y_m\right)= y_m\log\mu(x_m^\top\theta^*_m) + (1-y_m)\log(1-\mu(x_m^\top\theta^*_m)).
\label{eq:likelihood}
\end{align}
With direct computation, we have that the Fisher information matrix of $\theta_m$ is
\begin{align}
\label{eq:fisher}
\mathbf{E}_m = \mathbb E \left[\mu'(x_m^\top{\theta_m^*})x_mx_m^\top\right], \quad \forall m\in [M+1],
\end{align}
where the expectation with respect to $x_m$ is always under distribution $\mu_m$. The Fisher information matrix $\mathbf{E}_m$ can measure the information of $\theta_m$ we get from task $m$. %To prove a sharper bound, we use the Fisher information matrix $\mathbf{E}_m$ as the measure of the norm to construct the confidence set rather than using the covariance matrix of $x_m$, while a similar technique is used in \citet{faury2020improved} and \citet{abeille2021instance}. Thus, it's reasonable to assume that the Fisher information matrices are almost the same across all tasks. We also make a mild assumption that normalized $x_m$ is well distributed, thus we can use the observed Fisher information matrix to approximate the Fisher information matrix accurately. \han{this explanation seems too technical and hard to understand.} 
To perform efficient multi-task estimations, we assume the following to ensure that the Fisher information matrices are almost the same across tasks and the normalized data is well-distributed.
\begin{assumption}
\label{smp:var}
There exists a positive definite $d\times d$ matrix $\mathbf E$ and constants $C_1\leq 1\leq C_2$, such that 
\begin{align*}
C_1 \mathbf{E} \preceq \mathbf{E}_m \preceq C_2 \mathbf{E}, \quad \forall m\in [M+1].
\end{align*}
We also assume there exists $\rho_x>0$, such that $\mathbf{E}_m^{-1/2}\sqrt{\mu'( x_m^\top {\theta_m^*})}x_m$ are mean zero random variables with $\psi_2-$norm less than $\rho_x$. 
\end{assumption}

Assumption~\ref{smp:var} is a more general form of previous standard assumptions on multi-task learning \citep{du2020few,chen2022active,wang2023improved}, which state that either distribution $\mu_m$ or the covariance matrices $\EE[x_m x_m^\top]$ were identical across tasks. In particular, when all $\mu_m$ or $\EE[x_m x_m^\top]$ are the same for every task, we can set $\mathbf E=\mathbb E[x_mx_m^\top]$. It can be shown that Assumption \ref{smp:var} holds for $C_1=1/\kappa(L)$ and $C_2=1$, where $\kappa(L)$ is defined in Assumtpion~\ref{smp:bound}.  %Assume $\mathbf E$ is positive definite without loss of generality. Because of Assumption \ref{smp:bound}, $1/\kappa(L)\leq\mu'(x_m^\top\theta^*_m)\leq 1$. Thus, Assumption \ref{smp:var} holds for $C_1=1/\kappa(L), C_2=1$. Moreover, since $\kappa(L)$ scales exponentially, we can not get an optimal dependency on $\kappa$ if only assuming $\mathbb E[x_mx_m^\top]$ are the same. We remark that what we essentially need for multi-task RLHF is Assumption \ref{smp:bound} rather than the same covariance assumption as in \citepchen2022active} and \citepwang2023improved}.

\paragraph{Learning Objective.} Our goal is to find a near-optimal policy for the target task using data from source tasks and the target task. Given a policy $\pi$, its optimality under the task task is defined as
\begin{align*}
\mathrm{SubOpt}(\pi) = \mathbb E_{s\sim\rho}[r^*(s) - r^*(s, \pi(s))],
\end{align*}
where the expectation is taken over a fix distribution $\rho$ of $s$, $\pi^*(s) = \argmax_{a\in\mathcal{A}}r^*(s, a)$ represents the optimal action, and $r^*(s) = r^*(s, \pi^*(s))$ denotes the maximum reward for context $s$. We assume, without loss of generality, that $r^*(s)>0$ for all $s\in \mathcal S$. A policy $\pi$ is considered $\varepsilon-$optimal if $\mathrm{SubOpt}(\pi)\le\varepsilon$.

\section{Algorithm and Theory}

In this section, we present our main algorithms for solving multi-task RLHF problems, followed by theoretical guarantees. 

\subsection{Results for the Known Task Relevance Setting}

In this section, we present the multi-task algorithm when the task relevance coefficient $\nu$ defined in Assumption \ref{smp:lowdim} is known. 

To reduce the total sample size, we set uneven sample numbers for $M$ source tasks. Intuitively, we need to sample more data from the task $m$ that $|\nu_m|$ is large, which means this task is more relevant to the target task. We collect $n_m$ i.i.d. samples from source task $m$, with matrix form $(\mathbf X_m,y_m)$. Specifically, given the total sample size $N$, the sample size for each source task is computed as $\texttt{Active-Sample}(N, \nu)\in\mathbb Z^M$, where $n_m$ is almost proportional to $|\nu_m|$ and the $1/2M$ term is a technique to avoid small sample size,  
\begin{align}
\label{eq:duelingsamplesize}
\{n_m\}_{m=1}^M = \texttt{Active-Sample}(N, \nu) :=  \Big\{ \Big(\frac{|\nu_m|}{2\|\nu\|_1}+ \frac{1}{2M}\Big)N \Big\}_{m=1}^M. 
\end{align}

Using the prespecified regions $\Theta_m$, and the subspace assumption, it is natural to set the region of solution $(\theta_1,\dots,\theta_M)$ as 
\begin{align*}
\Theta_{\sou} = \left\{(\theta_1,\dots,\theta_M)| {\rm rank}({\rm span}\{\theta_1,\dots,\theta_M\})\leq k, \theta_m\in\Theta_m, \forall m\in[M]\right\}. 
\end{align*}

Given data, the low-rank region, and a properly chosen penalty term $\lambda>0$ (specified in Appendix \ref{appendix:pf:thm:known}), we can estimate $(\theta^*_1,\dots,\theta^*_M)$ as regularized MLE constrained in the region, 
\begin{align}
\label{eq:duelingmle}
(\hat\theta_1,\dots,\hat\theta_M)= \argmax_{\Theta_\sou} \sum_{m=1}^M\sum_{i=1}^{n_m}\ell\left(x_{m,i}^\top\theta_m, y_{m,i}\right) - \lambda\sum_{m=1}^M n_m\|\theta_m\|_2^2, 
\end{align}
where $\ell$ is defined in \eqref{eq:likelihood}. With the estimation $(\hat\theta_1,\dots,\hat\theta_M)$ in \eqref{eq:duelingmle} and the known task relevance coefficient $\nu = \{\nu_m\}_{m=1}^M$, the parameter of the target task is estimated by 
$\hat\theta=\sum_{i=1}^M \nu_m\hat\theta_m.$
% Define the expected Hessian matrix as 
% \begin{align}
% \mathbf H = \mathbf E + \lambda \mathbf I. 
% \end{align}
%Finally, we estimate $\mathbf{H}$ using 
%\begin{align}
%\mathbf{\Lambda} = \frac{1}{N}\sum_{m=1}^M \sum_{i=1}^{n_m} \mu'(x_{m,i}^\top\hat\theta_m)x_{m,i}x_{m,i}^\top+\lambda \mathbf I. 
%\end{align}
We can construct a confidence set for $\theta^*$ with proper parameter $\alpha>0$, that is 
\begin{align} \label{eq:confidence:set:known}
\mathcal{C} = \left\{\theta : \left\|\theta-\hat\theta\right\|_{\mathbf{\Lambda}}\leq \alpha\varepsilon\right\}, \quad \text{ where } \mathbf{\Lambda} = \frac{1}{N}\sum_{m=1}^M \sum_{i=1}^{n_m} \mu'(x_{m,i}^\top\hat\theta_m)x_{m,i}x_{m,i}^\top+\lambda \mathbf I ,
\end{align}
where $\mathbf{\Lambda}$ is an estimation of $\mathbf{H}:=\mathbf{E} + \lambda \mathbf{I}$. Based on the confidence set, we give policy $\pi$ as the policy that maximizes the pessimistic expected value \citep{zhu2023principled}, 
\begin{align}
\label{eq:duelingpolicy}
\pi = \argmax_{\pi'} \min_{\theta\in\mathcal{C}} \left\{\mathbb E_{s\sim\rho}[\phi(s, \pi'(s))^\top\theta]\right\}. 
\end{align}

\begin{algorithm}
\label{alg:dueling}
\caption{Multi-Task Low-Rank RLHF in Known Task Relevance Setting.}
\begin{algorithmic}[1]
\REQUIRE Total sample size $N$, task relevance $\nu = \{\nu_m\}_{m=1}^M$. 
\STATE Initialize $\lambda$ satisfying \eqref{eq:duelinglambda} and compute sample size $n_m$ via \eqref{eq:duelingsamplesize}. 
%\STATE Draw $n_m$ i.i.d samples from source task $m$, for all $m\in[M]$. 
\STATE Estimate $(\hat\theta_1,\dots,\hat\theta_M)$ by regularized MLE in \eqref{eq:duelingmle}. 
\STATE Let $\hat\theta=\sum_{i=1}^M \nu_m\hat\theta_m$ and construct the confidence set $\cC$ as \eqref{eq:confidence:set:known}.
%\STATE Get policy $\pi$ from \eqref{eq:duelingpolicy}
\RETURN $\pi = \argmax_{\pi'} \min_{\theta\in\mathcal{C}} \left\{\mathbb E_{s\sim\rho}[\phi(s, \pi'(s))^\top\theta]\right\}$. 
\end{algorithmic}
\end{algorithm}

% \textcolor{blue}{where to put, same notation B}
% \begin{assumption}
% There exists $B>0$, for all $s\in\mathcal {S}$, assume the optimal action is $a^*$, then 
% \begin{align}
% \|\phi(s, a^*)\|_{\mathbf{H}^{-1}} \leq B.
% \end{align}
% We have $\varepsilon \lesssim 1/B$. 
% \end{assumption}

\begin{comment}
We need the total sample size to satisfy 
\begin{align}
\label{eq:duelingsmaplesize}
N\gtrsim \max\left(\Tilde{\mathcal O}\left(\|\nu\|_1^2k(d+M)\log(1/\delta)^4\varepsilon^{-2}\right),
\rho_x^4M(d+\log(M/\delta))\right).
\end{align}
\end{comment}

Below, we give the sample complexity bound for Algorithm~\ref{alg:dueling}. See Appendix \ref{appendix:pf:thm:known} for a formal version of the theorem and its detailed proof.

\begin{theorem} \label{thm:known}
Suppose Assumptions \ref{smp:lowdim}, \ref{smp:bound}, and \ref{smp:var} hold. Then for any sufficiently small $\varepsilon > 0$, Algorithm~\ref{alg:dueling} outputs an $\varepsilon$-optimal policy with probability at least $1-\delta$ when 
\begin{align*}
N=\Tilde{\mathcal O}\left(C^*\|\nu\|_1^2k(d+M)\varepsilon^{-2}\right),
\end{align*}
where the coverage coefficeint $C^*$ is defined as
\begin{align}
\label{eq:duelingcover}
C^*=\left\|\mathbb E_{s\sim\rho}\left[\phi(s, \pi^*(s))\right]\right\|_{\mathbf{H}^{-1}}^2, \quad \mathbf H = \mathbf E + \lambda \mathbf I.
\end{align}
\end{theorem}

\begin{remark}
If sampled uniformly from $M$ source tasks, then $\|\nu\|_1^2$ will be replaced by $M\|\nu\|_2^2$, which is no less than $\|\nu\|_1^2$. The $\|\nu\|_1^2$ measures the imbalance of source tasks, the improvement $M\|\nu\|_2^2/\|\nu\|_1^2$ is large if source tasks are imbalance. 
\end{remark}

\begin{remark}
While $\kappa$ defined in \eqref{eq:kappa} can scale exponentially, we should consider the dependency on $\kappa$. The dependency on $\kappa$ is represented in $\mathbf{E}$ and $\mathbf{H}$. Specifically, the $\mu'(x_m^\top\theta^*_m)$ term in the definition \eqref{eq:fisher} of $\mathbf E_m$ is at least $1/\kappa(L)$. We find that the estimation error bound matches the results in \citet{abeille2021instance} and the sample complexity matches the results of learning a single task in \citet{zhu2023principled}. 
\end{remark}

\begin{proof}[\textbf{Proof Sketch of Theorem \ref{thm:known}}]
We first prove three auxiliary lemmas. The first is the core method to handle the log-likelihood function $\ell(t, y)$, that is to approximate it by a quadratic function in a bounded region (Lemma \ref{lem:duelingappro}). Using this lemma, we can change the form form MLE to a form similar to a square loss. Secondly, we have that the Hessian Matrix is concentrated around its mean (Lemma \ref{lem:duelinghessian}). Through the proof, we will apply this lemma many times to different kinds of data and the corresponding regularized likelihood function. The third is that the noise of a single task can be bounded with high probability(Lemma \ref{lem: duelingsingle}). Similar concentration bounds are presented as Theorem $1$ in \citet{faury2020improved}. Because the noise $\eta$ is bounded by $1$ and has a small variance, we should use Matrix Bernstein's inequality for rectangular matrices (Lemma \ref{lem:matrixbern}) to set up the concentration of the noise. Compared with previous works, we get a clean bound and a simple proof. 

To deal with noise from different source tasks uniformly, we define 
\begin{align}
\label{eq:duelingnoisenotation}
g_m={\mathbf{\hat H}_m}^{-\frac{1}{2}}\mathbf X_m^\top \eta_m, 
\end{align}
as the normalized noise, and $\mathbf G$ as a $d\times M$ matrix with columns $g_m$. Apply a union bound, the norm of $g_m$ can be uniformly bounded with high probability (Lemma \ref{lem:duelingnorm}),
\begin{align}
\|g_m\|_2\leq \Tilde{\mathcal O}\left(\sqrt{d}\log(1/\delta)\right),\quad \forall m\in [M].
\end{align}
By truncating $g_m$ and applying Matrix Bernstein's inequality again, we can bound $\|\mathbf G\|$ (Lemma \ref{lem:duelingnoise}),
\begin{align}
\|\mathbf G\|\leq \Tilde{\mathcal O}\left(\sqrt{d+M}\log(1/\delta)^2\right). 
\end{align}

Now, we can prove the main result that the estimation error of the weighted sum of $\theta^*_m$ is small, with high probability (Lemma \ref{thm:duelinggeneral}). We first use the approximation of the log-likelihood function. Then bound the noise term and the regularization term. Finally, we apply Cauchy-Schwarz inequality to complete the proof. Specifically, we have
\begin{align*}
\left\|\hat\theta-\theta^*\right\|_\mathbf{H}^2=\left\|\sum_{m=1}^M\nu_m(\hat\theta_m-\theta^*_m)\right\|_\mathbf{H}^2\leq\left(\sum_{m=1}^M \frac{\nu_m^2}{n_m}\right)\Tilde{\mathcal O}\left(k(d+M)\log(1/\delta)^4\right)\leq\Tilde{\mathcal O}\left(\frac{\varepsilon^2}{C^*}\right) .
\end{align*}

Results in Lemma \ref{thm:duelinggeneral} set up guarantees for the confidence set.  Theorem \ref{thm:known1} is then proved in Lemma \ref{thm:duelingpolicy}, by first ensuring the accuracy of $\mathbf\Lambda$ and then bounding the gap using the coverage $C^*$. 
\end{proof}

\subsection{Results for the Unknown Task Relevance Setting}

Based on the results in the previous section, now we present the algorithm when $\nu$ is unknown. Since there may exist many $\nu$ satisfying $\theta^*=\sum \nu_m \theta^*_m$, we fix $\nu^*$ to have minimum $\ell_1-$norm and thus have the minimum sample complexity under known $\nu$ setting, namely
\begin{align} \label{eq:nu:star}
\nu^* = \argmin_{\theta^*=\sum \nu_m \theta^*_m}\|\nu\|_1.
\end{align}
The algorithm can be divided into two steps, estimating $\nu^*$ and estimating $\theta^*$. Under this setting, we need access to sample data from the target task. In the first step, we use a small number of
data from the source tasks and target tasks to estimate $(\theta^*_1, \dots, \theta^*_M)$ and $\theta^*$. The sample size of each source task is determined by a prior of $\nu^*$. Using these estimations, we can get an estimation of $\nu^*$ by Lasso Programming \citep{Tibshirani1996lasso}. In the second step, we determine the sample size of each source task based on the estimated $\nu^*$ similar to the algorithm of the known $\nu$ setting. At the same time, these estimations provide a learned subspace for estimating $\theta^*$. Estimating $\theta^*$ based on the learned subspace reduces the sample size needed for the target task. 

\paragraph{Estimating $\nu^*$.} Without loss of generality, we assume access to a prior, $\nu^0 = \{\nu_m^0\}_{m=1}^M \in\mathbb R^M$ satisfying $\|\nu^0\|_1=1$, for the $\nu^*$ defined in \eqref{eq:nu:star}. If we have no information about the relevance of each task, we can set each element of $\nu^0$ equal to $1/M$, which represents a non-informative prior. The prior $\nu^0$ is accurate when $\nu^0$ is just $\nu^*$ multiplied by a scalar.

Firstly, we set the total number of samples of source tasks to be $N^\pre_\sou$, and allocate the sample size of $m$ task similar to methods in the previous section (cf. \eqref{eq:duelingsamplesize}) based on the prior $\nu^0$, that is $(n^\pre_1, \dots, n^\pre_M)=\texttt{Active-Sample}(N_\sou^\pre, \nu^0)$. We sample $n^\pre_m$ samples i.i.d from source task $m$ and let $\mathbf X_m^\pre$ and $y_m^\pre$ to be the matrix of features of samples and vector of feedback from task $m$. Same as that in the known $\nu$ setting, given data and a proper penalty term $\lambda^\pre_\sou$ (all proper parameters are specified in Appendix \ref{appendix:pf:thm:unknown}), we estimate $(\theta^*_1,\dots,\theta^*_M)$ as regularized MLE, 
\begin{align}
\label{eq:activemle1}
(\hat\theta_1^\pre,\dots,\hat\theta_M^\pre)= \argmax_{\Theta_\sou} \sum_{m=1}^M\sum_{i=1}^{n^\pre_\sou}\ell\left((x^\pre_{m,i})^\top\theta_m, y^\pre_{m,i}\right) - \lambda_\sou^\pre\sum_{m=1}^M n^\pre_m\|\theta_m\|_2^2. 
\end{align}

Secondly, we sample $n^\pre$ i.i.d samples from the target task. With $\mathbf X^\pre$ and $y^\pre$ as the sample and feedback, and with a proper penalty term $\lambda^\pre$, we can estimate $\theta^*$ as regularized MLE, 
\begin{align}
\label{eq:activemlet1}
\hat\theta^\pre = \argmax_\Theta \sum_{i=1}^{n^\pre}\ell\left((x^\pre_{i})^\top\theta , y^\pre_{i}\right) - n^\pre\lambda^\pre \|\theta\|_2^2. 
\end{align}

Finally, we estimate $\nu^*$. Before that, we estimate   $\mathbf{H}:=\mathbf{E}+\min(\lambda^\pre_\sou, \lambda^\pre, \lambda_\sou, \lambda) \mathbf I$ (where $\lambda^\pre_\sou, \lambda^\pre,\lambda_\sou, \lambda$ are prespecified proper regularization terms). With a slight abuse of notation, we also use $\mathbf{\Lambda}$ to denote the estimation. 
\begin{align*}
\mathbf{\Lambda} = \frac{1}{N^\pre}\sum_{m=1}^M \sum_{i=1}^{n^\pre_m} \mu'((x_{m,i}^\pre)^\top\hat\theta^\pre_m)x_{m,i}^\pre(x_{m,i}^\pre)^\top+\min(\lambda^\pre_\sou, \lambda^\pre, \lambda_\sou, \lambda) \mathbf I. 
\end{align*}
Now, based on the estimation of $\theta^*_m$ and $\theta^*$, we can estimate $\nu^*$. We set a proper parameter $R$ satisfying $
R\geq \sum_{m=1}^M (\nu^*_m)^2/|\nu^0_m|
$ as a bound, 
and define a region $\Theta_\nu$ of $\nu$ ensuring that the true $\nu^*$ lies in the region, that is $\nu^*\in\Theta_\nu$. Since the sample complexity in Theorem \ref{thm:known} is related to $\|\nu\|_1$, thus we should penalize the loss with $\ell_1-$norm of $\nu$. Specifically, we use Lasso Programming \citep{Tibshirani1996lasso} with a proper penalty $\beta$, 
\begin{align}
\label{eq:activelatent}
\hat\nu^\pre = \argmin_{\nu\in \Theta_\nu} \frac{1}{2}\bigg\|\sum_{m=1}^M\nu_m\hat\theta_m^\pre-\hat\theta^\pre\bigg\|_{\mathbf\Lambda}^2 + \beta\|\nu\|_1, 
\quad
\Theta_\nu = \bigg\{\nu: \sum_{m=1}^M\frac{(\nu_m)^2}{|\nu^0_m|}\leq R\bigg\}.
\end{align}

\paragraph{Estimating $\theta^*$.} We set the total sample number of source tasks as
$N_\sou$, and compute sample size for each source task as follows, based on the estimation $\nu^\pre$, we can get the sample size $(n_1, \dots, n_M)=\texttt{Active-Sample}(N_\sou, \hat\nu^\pre)$. After sampling $n_m$ samples $(\mathbf X_m, y_m)$ from source task $m$, given a proper $\lambda_\sou$, we estimate $(\theta^*_1,\dots,\theta^*_M)$ as regularized MLE, 
\begin{align}
\label{eq:activemle2}
(\hat\theta_1,\dots,\hat\theta_M)= \argmax_{\Theta_\sou} \sum_{m=1}^M\sum_{i=1}^{n_m}\ell\left(x_{m,i}^\top \theta_m , y_{m,i}\right) - \lambda_\sou\sum_{m=1}^M n_m\|\theta_m\|_2^2. 
\end{align}

Lastly, we sample $n$ samples $\{(x_i, y_i)\}_{i=1}^n$ with matrix form $(\mathbf X, y)$ from the target task and estimate $\theta$ in the estimated latent space $\hat{\mathcal  S}$ provided by $(\hat\theta_1, \dots, \hat\theta_M)$. With proper $\lambda$, 
\begin{align}
\label{eq:activemlet2}
\hat\theta= \argmax_{\hat {\mathcal S}\cap \Theta} \sum_{i=1}^{n}\ell\left(x_{i}^\top\theta, y_{i}\right) - n\lambda \|\theta\|_2^2, \text{ where } \hat{\mathcal  S} = {\rm span}\{\hat\theta_1, \dots, \hat\theta_M\}. 
\end{align}

\begin{comment}
Define 
\textcolor{blue}{same notation ......}
\begin{align}
\mathbf H = \mathbf E + \min(\lambda^\pre_\sou, \lambda^\pre, \lambda_\sou, \lambda)\mathbf I. 
\end{align}
\end{comment}

Same as in the previous section, we construct a confidence set for $\theta^*$  with proper parameter $\alpha>0$ and give the policy $\pi$. Precisely,   
\begin{align}
\label{eq:activepolicy}
\pi = \argmax_{\pi'} \min_{\theta\in\mathcal{C}} \left\{\mathbb E_{s\sim\rho}[\phi(s, \pi'(s))^\top\theta]\right\}, \text{ where }\mathcal{C} = \big\{\theta : \big\|\theta-\hat\theta\big\|_{\mathbf{\Lambda}}\leq \alpha\varepsilon\big\}.  
\end{align}

\begin{comment}
\textbf{Sample complexity and proper parameters.} 
\begin{align}
\begin{split}
N^\pre_\sou &\gtrsim \max\left(\Tilde{\mathcal O}\left(C_\theta^2M(R/\|\nu^*\|_1^2)k(d+M)\log(1/\delta)^4\right), \rho_x^4M(d+\log(M/\delta))\right),\\
n^\pre &\gtrsim \max\left(C_\theta^2(M/\|\nu^*\|_1^2)d\log(d/\delta)^2, \rho_x^4(d+\log(1/\delta))\right),\\
N_\sou&\gtrsim \max\left(\Tilde{\mathcal O}\left(\|\nu^*\|_1^2k(d+M)\log(1/\delta)^4\varepsilon^{-2}\right),
\rho_x^4M(d+\log(M/\delta))\right), \\
n &\gtrsim \max\left(k\log(d/\delta)^2\varepsilon^{-2}, \rho_x^4(d+\log(1/\delta))\right).\\
\end{split}
\end{align}
We can find that the sample size of source tasks $N^\pre_\sou$ and target task $n^\pre$ in the first step is not related to $\varepsilon$. With these additional samples, the sample size of source tasks $N_\sou$ needed in the second step matches the results in the known $\nu$ setting. The main term of sample size of target task $n$ in the second step is linear in $k$ and only scales with the logarithm of dimension $d$. 

The proper penalties $\lambda^\pre_\sou, \lambda^\pre, \lambda_\sou, \lambda$ are chosen satisfying
\begin{align}
\label{eq:activelambda}
B_x^2/d\lesssim N^\pre_\sou\lambda^\pre_\sou/M, n^\pre \lambda^\pre, N_\sou \lambda_\sou/M, n\lambda\lesssim k/B_\theta^2,
\end{align}
and the penalty term for the Lasso Programming satisfies 
\begin{align}
\label{eq:activebeta}
\|\nu^*\|_1/C_\theta^2M\lesssim \beta \lesssim \|\nu^*\|_1/C_\theta^2M.
\end{align}

\end{comment}

\begin{algorithm}
\label{alg:unknown}
\caption{Active Multi-Task Low-Rank RLHF.}
\begin{algorithmic}[1]
\REQUIRE Total sample size $N$, prior task relevance $\nu^0$. 
\STATE Initialize $N^\pre_\sou, n^\pre, N_\sou, n$ satisfying \eqref{eq:activesmaplesize} $\lambda^\pre_\sou, \lambda^\pre, \lambda_\sou, \lambda$ satisfying \eqref{eq:activelambda}, $\beta$ satisfying \eqref{eq:activebeta}. \\
\textcolor{bluee}{\texttt{Phase 1: Estimating $\nu^*$:}}
\STATE Let $\{n^\pre_m\}_{m=1}^M=\texttt{Active-Sample}(N_\sou^\pre, \nu^0)$ and estimate $(\hat\theta_1^\pre,\dots,\hat\theta_M^\pre)$ via \eqref{eq:activemle1}. 
\STATE Estimate $\hat\theta^\pre$ and $\hat\nu^\pre$ via \eqref{eq:activemlet1} and \eqref{eq:activelatent}, respectively. \\
%\STATE Estimate $\hat\nu^\pre$ via %\eqref{eq:activelatent}. \\
\textcolor{bluee}{\texttt{Phase 2: Estimating $\theta^*$:}}
\STATE Let $\{n_m\}_{m=1}^M=\texttt{Active-Sample}(N_\sou, \hat\nu^\pre)$ and estimate $(\hat\theta_1,\dots,\hat\theta_M)$ via \eqref{eq:activemle2}. 
\STATE Estimate $\hat\theta$ by MLE via \eqref{eq:activemlet2}. 
%\STATE Get policy $\pi$ from \eqref{eq:activepolicy}
\RETURN Policy $\pi$ from \eqref{eq:activepolicy}. 
\end{algorithmic}
\end{algorithm}

Now, we can state the main theorem. See Appendix \ref{appendix:pf:thm:unknown} for a formal version of the theorem and its detailed proof.

\begin{theorem}[Informal] \label{thm:unknown}
Let Assumptions \ref{smp:lowdim}, \ref{smp:bound}, and \ref{smp:var} hold. 
 Then for any sufficiently small $\varepsilon > 0$, Algorithm \ref{alg:unknown} outputs an $\varepsilon$-optimal policy with probability at least $1-\delta$ when
\begin{align}
\begin{split}
\label{eq:activesmaplesize}
N_\sou=\Tilde{\mathcal O}\left(C^*\|\nu^*\|_1^2k(d+M)\varepsilon^{-2}\right), \quad
n = \Tilde{\mathcal O}\left(C^*k\varepsilon^{-2}\right), 
\end{split}
\end{align}
and $N^\pre_s, n^\pre$ above certain thresholds that are unrelated to $\varepsilon$ and only depend on other problem-dependent parameters such as $d$ and $k$. Here the coverage $C^*$ is defined as follows if we abuse notation $\mathbf{H}$,   
\begin{align}
\label{eq:activecover}
C^*=\left\|\mathbb E_{s\sim\rho}\left[\phi(s, \pi^*(s))\right]\right\|_{\mathbf{H}^{-1}}^2, \quad \mathbf H = \mathbf E + \min(\lambda^\pre_\sou, \lambda^\pre, \lambda_\sou, \lambda) \mathbf I.
\end{align} 
\end{theorem}

\begin{remark}
We remark following good properties possessed by multi-task learning. 
\begin{itemize}
\item We can find that the sample complexity of $N_\sou$ is the same as that in the case when $\nu$ is known. The additional sample complexity used to estimate $\nu^*$ is unrelated to $\varepsilon$, which is just lower-order terms. 
\item For the same reason stated in the previous section, if samples are drawn uniformly from each task, the sample complexity should be much larger. 
\item If we directly sample from the target task, since we do not have the learned low-rank subspace, the sample complexity of $n$ should be $C^*d\varepsilon^{-2}$. With the aid of an estimated feature extractor, the sample complexity of the target task only scales linearly in $k$. 
\end{itemize}
\end{remark}

\begin{comment}

To get the precise sample complexity in the estimating $\nu^*$ step, we need the following term. It is similar to the diverse task assumption mentioned in \citet{du2020few}, \citet{tripuraneni2021provable} and \citet{wang2023improved}. We define a term $C_\theta$ and state it more explicitly. Intuitively, $C_\theta$ is a measure of how well the linear combinations of $\theta^*_m$ cover all directions in the space spanned by $\theta^*_m$. What we need is that any $\theta'$ with unit $\mathbf H-$norm ($\mathbf H$ is defined in \eqref{eq:activecover}) in the space can be represented by a linear combination of $\theta^*_m$ with small coefficients. The term $C_\theta$ only affects sample complexity in the step that estimates $\nu^*$, which is independent of $\varepsilon$. \han{move this part to the appendix? give a detailed comparison with the previous assumption to show that our assumption is more general}
\begin{definition}
\label{asp:wellrepresent}
Let $C_\theta>0$ be the minimal number that makes the following true. For any $\theta'\in{\rm span}\{\theta^*_1,\dots,\theta^*_M\}$ with norm $\|\theta'\|_\mathbf{H}=1$, there exists $\alpha\in \mathbb R^M$ with norm $\|\alpha\|_2\leq C_\theta$, such that $\theta' = \sum_{m=1}^M \alpha_m\theta^*_m$.
\end{definition}

\end{comment}

Detailed analysis for $N^\pre_\sou, n^\pre$ can be found in Remark \ref{remark:activesamplepre} in Appendix \ref{appendix:pf:thm:unknown}.

\begin{proof}[\textbf{Proof Sketch of Theorem \ref{thm:unknown}}]
The proof is separated into two parts according to Algorithm \ref{alg:unknown}. 

\textbf{Estimating $\nu^*$.} We first bound the error of source tasks (Lemma \ref{lem:activesource1}) and the target task (Lemma \ref{lem:activetarget1}) by $\varepsilon_\theta$, which is defined as  
\begin{align}
\label{eq:activectheta}
\varepsilon_\theta = \frac{\|\nu^*\|_1}{\sqrt{M}C_\theta}. 
\end{align}
Here, $C_\theta$ is defined in Definition \ref{asp:wellrepresent}. The techniques used are similar to that of the proof of Lemma \ref{thm:duelinggeneral}. Based on the error bound in Lemma \ref{lem:activesource1}, we can guarantee that $\mathbf \Lambda$ is a good approximation of $\mathbf H$ with high probability (Lemma \ref{lemma:activeLambda}). 

After that, we turn to give two good properties of $\hat\nu^\pre$. The first is low $l_1-$norm that can be used to bound the sample complexity in the next step (Lemma \ref{lem:activesamplecom}), namely $\|\hat\nu^\pre\|_1\lesssim \|\nu^*\|_1$. The second states that $\hat\nu^\pre$ can represent the task relevance of source tasks well (Lemma \ref{lem:activetilde}). Specifically, there exists $\Tilde\nu\in \mathbb R^M$ very close to $\hat\nu^\pre$, such that $\Tilde\nu$ precisely represents the proportion of source tasks,  
\begin{align*}
\sum_{m=1}^M\Tilde\nu_m\theta_m^* = \theta^*. 
\end{align*}
The proof of this lemma is based on the term in Definition \ref{asp:wellrepresent}. 

\textbf{Estimating $\theta^*$.} 
With $\Tilde\nu$ defined in Lemma \ref{lem:activetilde}, we can define the auxiliary estimation of $\theta^*$,  
\begin{align}
\label{eq:activetildetheta}
\Tilde\theta = \sum_{m=1}^M \Tilde\nu_m\hat\theta_m. 
\end{align}
The term $\Tilde{\theta}$ is close to $\theta^*$ (Lemma \ref{lem:activesource2}), and, based on that result, we bound the error of estimating $\theta^*$ (Lemma \ref{thm:activemain}). To be specific, with high probability, 
\begin{align*}
\left\|\Tilde\theta-\theta^*\right\|_\mathbf{H}\lesssim \frac{\varepsilon}{\sqrt{C^*}}, \quad \left\|\hat\theta-\theta^*\right\|_{\mathbf{H}}\lesssim \frac{\varepsilon}{\sqrt{C^*}}. 
\end{align*}
During the proof, we need another approximation of the log-likelihood function from another direction (Lemma \ref{lem:activeappro}), different from that in Lemma \ref{lem:duelingappro}. Finally, the main result is proved in Lemma \ref{thm:activepolicy} using similar techniques from the proof of Lemma \ref{thm:duelingpolicy}. 
\end{proof}

\section{Conclusion}

In this paper, we study the power of active multi-task learning in RLHF. We formulate the RLHF problem as a contextual dueling bandit problem, assuming a common linear representation. Our work demonstrates that the sample complexity of source tasks in multi-task RLHF can be decreased by considering task relevance and assigning different sample sizes to source tasks based on their relevance. We develop an algorithm to estimate task relevance using additional data and then learn a policy. We prove the sample complexity of the proposed algorithm to output an $\varepsilon-$optimal policy. 

There are two directions for future investigation. First, it is important to empirically demonstrate the effectiveness of active multi-task learning in RLHF with large language models. Also, it will be interesting to extend active multi-task learning to other tasks.

\paragraph{Acknowledgements.} Ruitao Chen is partially supported by the elite undergraduate training program of the School of Mathematical Sciences, Peking University. 

\newpage
\bibliographystyle{ims}
\bibliography{graphbib}

\newpage

\appendix

\section{Notations}

In this section, we define the notation of mean zero noise and the Hessian Matrix of the regularized likelihood that we use through the proof. 

We define the mean zero noise as $\eta_{m,i} = y_{m,i}-\mu(x_{m,i}^\top \theta^*_m)$. In the following proofs, we always use $\eta$ to denote noise or vector of noise corresponding to data feature $x$ with the same subscript and superscript. The noise has the following properties: 
\begin{align*}
|\eta_{m,i}|\leq 1, \quad \mathbb E [\eta_{m,i}^2] = \mu'(x_{m,i}^\top \theta^*_m).
\end{align*}

Define the Hessian Matrix of the regularized likelihood as  
\begin{align*}
\hat{\mathbf{H}}_m = \sum_{i=1}^{n_m} \mu'(x_{m,i}^\top\theta^*_m) x_{m,i}x_{m,i}^\top+n_m\lambda  \mathbf{I}. 
\end{align*}
We will also define $\mathbf{\hat H}$ with different subscripts and superscripts denoting the Hessian Matrix of corresponding data in the rest of the proof. 

\section{Proof of Theorem~\ref{thm:known}} \label{appendix:pf:thm:known}

We provide the formal version of Theorem~\ref{thm:known}.

\begin{theorem}[Formal Version of Theorem~\ref{thm:known}]
\label{thm:known1}
Suppose Assumptions \ref{smp:lowdim}, \ref{smp:bound}, and \ref{smp:var} hold. Then for sufficiently small $\varepsilon > 0$, with probability at least $1-\delta$, Algorithm \ref{alg:dueling} outputs an $\varepsilon$-optimal policy when 
\begin{align}
N\gtrsim\max\left(\Tilde{\mathcal O}\left(C^*\|\nu\|_1^2k(d+M)\log(1/\delta)^4\varepsilon^{-2}\right), \rho_x^4M(d+\log(M/\delta))\right),
\end{align}
where the coverage $C^*$ is defined as
\begin{align}
\label{eq:duelingcoverage}
C^*=\left\|\mathbb E_{s\sim\rho}\left[\phi(s, \pi^*(s))\right]\right\|_{\mathbf{H}^{-1}}^2, \quad \mathbf H = \mathbf E + \lambda \mathbf I,
\end{align}
and proper parameters satisfy 
\begin{align}
\label{eq:duelinglambda}
B_x^2/d\lesssim N\lambda/M\lesssim k/B_\theta^2, \quad \alpha=\mathcal O\left(1/\sqrt{C^*}\right). 
\end{align}
\end{theorem}

\subsection{Auxiliary Lemmas}

\begin{lemma}[Approximation of log-likelihood]
\label{lem:duelingappro}
Assume real numbers $|t|\leq L_1, |t^*|\leq L_2, y\in \{0,1\}$, where $L_1\geq \max\{L_2,1\}$, then
\begin{align*}
\ell(t, y)\leq \ell(t^*, y)+(y-\mu(t^*))(t-t^*)-\frac{1}{20L_1\kappa(L_2)}(t-t^*)^2.
\end{align*}
\end{lemma}

\begin{proof}
Without loss of generality, we may assume 
that $t\geq t^*$, otherwise replace $t, t^*$ by $-t, -t^*$. Now we fix $t^*,y$ and denote
\begin{align*}
f(t) = \ell(t, y)- \ell(t^*, y)-(y-\mu(t^*))(t-t^*)+\frac{1}{20L_1\kappa(L_2)}(t-t^*)^2.
\end{align*}
The derivative of $f$ can be expressed as
\begin{align*}
f'(t) = (y-\mu(t))-(y-\mu(t^*))+\frac{1}{10L_1\kappa(L_2)}(t-t^*) = \int_{t^*}^t g(s) {\rm d s},
\end{align*}
where $g(s)=-\mu'(s)+1/(10L_1\kappa(L_2))$. For $0\leq \varepsilon\leq 1$, we have 
\begin{align*}
g(t^*+\varepsilon) 
= -\mu'(t^*+\varepsilon) +\frac{1}{8L_1\kappa(L_2)}
\leq -\frac{1}{e\kappa(L_2)} +\frac{1}{10L_1\kappa(L_2)}\leq -\frac{1}{5\kappa(L_2)}.
\end{align*}
Which means that for $t^*\leq t\leq t^*+1$, we have $g(t)\leq 0$, and thus, $f'(t)\leq 0$. For $t^*+1\leq t\leq L_2$, We obtain
\begin{align*}
f'(t) = \int_{t^*}^{t^*+1}g(s){\rm d}s + \int_{t^*+1}^{t}g(s){\rm d}s \leq -\frac{1}{5\kappa(L_2)} + (2L_1)\frac{1}{10L_1\kappa(L_2)}\leq 0.
\end{align*}
Thus, $f(t)\leq f(t^*)=0$, for all $t^*\leq t\leq L_1$. 
\end{proof}

\begin{lemma}[Hessian estimation]
\label{lem:duelinghessian}
Since $n_m\gtrsim \rho_x^4(d + \log(M/\delta)), \forall m\in [M]$, we have, with probability at least $1-\delta/8$, 
\begin{align*}
\begin{split}
0.9C_1 \mathbf{H} \preceq \frac{1}{n_m} \hat{\mathbf{H}}_m \preceq 1.1C_2 \mathbf{H}, \qquad \forall m\in [M].
\end{split}
\end{align*}
\end{lemma}

\begin{proof}
For all $m\in[M]$, by assumption, 
$\mathbf{E}_m^{-\frac{1}{2}}\sqrt{\mu'( x_{m,i}^\top \theta^*_m)} x_{m,i}$ are mean zero random variables with $\psi_2-$norm less than $\rho_x$. Apply Lemma \ref{lem:varianceestimation} to above random variables and denote $\hat{\mathbf{E}}_m=\sum_{i} \mu'({\theta^*_m}^\top x_{m,i})x_{m,i}x_{m,i}^\top$, 
\begin{align*}
\mathbb P\left(0.9 C_1\mathbf{H} \preceq  \frac{1}{n_m}\hat{\mathbf{H}}_m \preceq 1.1C_2 \mathbf{H}\right)\geq \mathbb P\left(0.9 \mathbf{E}_m\preceq \frac{1}{n_m}\hat{\mathbf{E}}_m \preceq 1.1 \mathbf{E}_m\right)\geq 1-\delta/8M. 
\end{align*}
The lemma is proved by taking a union bound over $m\in[M]$.
\end{proof}

\begin{lemma}[Concentration of single task]
\label{lem: duelingsingle}
Assume $\{x_i\}_{i=1}^n$ are samples of a single task satisfying Assumption \ref{smp:bound} and \ref{smp:var}, with corresponding noise $\eta_i$. Let $\sigma^2_i$ to be the variance of noise $\sigma^2_i = {\rm Var}(\eta_i)$, and 
\begin{align*}
\mathbf{H}_x = \sum_{i=1}^n \sigma^2_ix_ix_i^\top+n\lambda \mathbf I, 
\end{align*}
where $n\lambda\gtrsim B_x^2/d$. 
Then with probability at least $1-\delta$, the following inequality holds, 
\begin{align*}
\left\|\sum_{i=1}^n\eta_i\mathbf{H}_x^{-\frac{1}{2}}x_i\right\|_2 \leq \mathcal O\left(\sqrt{d}\log(d/\delta)\right). 
\end{align*}
\end{lemma}

\begin{proof}
\textbf{Step 1: Checking prerequisites of Matrix Bernstein's inequality.} Let $A_i=\eta_i\mathbf{H}_x^{-\frac{1}{2}}x_i$ to be a $d\times 1$ matrix, we have $A_1,\dots,A_n$ are independent and mean zero. $\|A_i\|$ is bounded by
\begin{align}
\label{eq: gmbound}
\|A_i\| = \left\|\eta_{i}\mathbf{H}_x^{-\frac{1}{2}}x_{i}\right\|\leq \left\|\mathbf{H}_x^{-\frac{1}{2}}x_{i}\right\| \leq B_x/\sqrt{n\lambda}\leq \mathcal O(\sqrt{d}). 
\end{align}
Because 
\begin{align*}
\begin{split}
\sum_{i=1}^{n}\mathbb E A_i A_i^\top=\mathbf{H}_x^{-\frac{1}{2}}\left(\sum_{i=1}^{n}(\mathbb E \eta_{i}^2) x_{i} x_{i}^\top\right) \mathbf{H}_x^{-\frac{1}{2}}\preceq \mathbf I. 
\end{split}
\end{align*}
We also have
\begin{align*}
\begin{split}
\left\|\sum_{i=1}^{n}\mathbb E A_i^\top A_i\right\| = {\rm tr}\left(\sum_{i=1}^{n}\mathbb E A_i^\top A_i\right) = {\rm tr}\left(\sum_{i=1}^{n}\mathbb E A_iA_i^\top\right)\leq d, \quad \left\|\sum_{i=1}^{n}\mathbb E A_i A_i^\top\right\| \leq 1. 
\end{split}
\end{align*}
\textbf{Step 2: Applying Matrix Bernstein's inequality.} Take $K=\mathcal O(\sqrt{d})$ and $\sigma^2\leq d$ in Lemma \ref{lem:matrixbern}. Let $t=\mathcal O(\sqrt{d}\log(d/\delta))$. This completes the proof.  
\end{proof}

\subsection{Concentration of Noise of Multi-Task}

\begin{lemma}[Concentration of the norm of $g_m$]
\label{lem:duelingnorm}
We have, with probability at least $1-\delta/(16N)$, 
\begin{align*}
\|g_m\|_2\leq \Tilde{\mathcal O}\left(\sqrt{d}\log(1/\delta)\right),\quad \forall m\in [M],
\end{align*}
where $g_m$ is defined in \eqref{eq:duelingnoisenotation}. 
\end{lemma}

\begin{proof}
Fix $m\in M$ and use results in Lemma \ref{lem: duelingsingle}. Easy to check all the assumptions. We change the probability $\delta$ to $\delta/(16MN)$ and get
\begin{align*}
\mathbb P(\|g_m\|_2)\leq \delta/(16MN). 
\end{align*}
We complete the proof by taking a union bound. 
\end{proof}

\begin{lemma}[Concentration of $\|\mathbf G\|$]
\label{lem:duelingnoise}
We have, with probability at least $1-\delta/8$, 
\begin{align*}
\|\mathbf G\|\leq \Tilde{\mathcal O}\left(\sqrt{d+M}\log(1/\delta)^2\right), 
\end{align*}
where $\mathbf G$ is defined in \eqref{eq:duelingnoisenotation}. 
\end{lemma}

\begin{proof}
\textbf{Step 1: Truncation.} Let $\mathcal E$ to be the event of the statement in Lemma \ref{lem:duelingnorm}. To prove the theorem, we need to assume event $\mathcal E$ happens and use Matrix Bernstein's inequality conditioned on that event. 

We can bound $g_m$ by \eqref{eq: gmbound}, which gives
\begin{align*}
\|g_m\|_2\leq \sum_{i=1}^{n_m}\left\|\eta_{m,i}{ \mathbf{\hat H}_m}^{-\frac{1}{2}}x_{m,i}\right\|\leq N\mathcal O\left(\sqrt{d}\right).
\end{align*}
Denote 
\begin{align*}
p = 1-\mathbb P(\mathcal E),\quad\mu_m = \mathbb E\left[g_m|\mathcal E\right], \quad \bar g_m = g_m - \mu_m.
\end{align*}
Then 
\begin{align}
\begin{split}
\label{eq:duelingnoise}
\mathbb E\left[\bar g_m\bar g_m^\top|\mathcal E\right] \preceq& \mathbb E\left[g_mg_m^\top|\mathcal E\right] 
=  \left(\mathbb E\left[g_mg_m^\top\right]-p\mathbb E\left[g_mg_m^\top|\mathcal E^c\right]\right)/(1-p)
\preceq 2\mathbb E\left[g_mg_m^\top\right]
\preceq  2\mathbf I.
\end{split}
\end{align}

\textbf{Step 2: Checking prerequisites of Matrix Bernstein's inequality.} In the following proofs, probabilities and expectations are all conditioned on event $\mathcal E$. 

For $m\in [M]$, let $\mathbf A_m$ to be $d\times M$ random matrix with $m$th column equals to $\bar g_m = g_m-\mu_m$ and other columns all zeros. $\mathbf A_1,\dots, \mathbf A_M$ are independent and mean zero. $\|\mathbf A_m\|$ is bounded by
\begin{align*}
\|\mathbf A_m\| = \|g_m-\mu_m\|_2 \leq \|g_m\|_2+\|\mathbb Eg_m\|_2\leq \Tilde{\mathcal O}\left(\sqrt{d}\log(1/\delta)\right). 
\end{align*}
Because of \eqref{eq:duelingnoise}, it follows that
\begin{align*}
\begin{split}
\left\|\sum_{m=1}^M\mathbb E \mathbf A_m^\top \mathbf A_m\right\| &= \max_m\left\{\mathbb E \bar g_m^\top \bar g_m\right\} = \max_m\left\{{\rm tr}(\mathbb E \bar g_m\bar g_m^\top) \right\}\leq 2d, \\
\left\|\sum_{m=1}^M\mathbb E \mathbf A_m \mathbf A_m^\top\right\| &= \left\|\sum_{m=1}^M\mathbb E \bar g_m \bar g_m^\top\right\| \leq 2M. 
\end{split}
\end{align*}

\textbf{Step 3: Applying Matrix Bernstein's inequality.} Let $K=\Tilde{\mathcal O}(\sqrt{d}\log(1/\delta))$ and $\sigma^2 = 2(d+M)$in Lemma \ref{lem:matrixbern}. We take $t = \Tilde{\mathcal O}(\sqrt{d+M}\log(1/\delta)^2)$, then $
\mathbb P\left(\|\mathbf G\|\geq t|\mathcal E\right)\leq \delta/16$. Combining results in Lemma \ref{lem:duelingnorm}, we get 
\begin{align*}
\mathbb P\left(\|\mathbf G\|\geq t\right)\leq \mathbb P\left(\|\mathbf G\|\geq t|\mathcal E\right) + \mathbb P\left(\mathcal E^c\right)\leq \delta/8. 
\end{align*}
\end{proof}

\subsection{Estimation Error of Multi-Task}

\begin{lemma}[Estimation error of multi-task]
\label{thm:duelinggeneral}
Let Assumptions \ref{smp:lowdim}, \ref{smp:bound}, and \ref{smp:var} hold. For $m\in[M]$, sample $n_m$ samples from source task $m$, and let $\mathbf X_m, y_m$ be the matrix of data features and feedback. 
The estimation of parameters is via \eqref{eq:duelingmle}, 
where $\lambda$ satisfies $B_x^2/d\lesssim N\lambda/M \lesssim k/B_\theta^2$. The matrix $\mathbf{H} = \mathbf E+\lambda \mathbf I$. If the sample size $n_m\gtrsim \rho_x^4(d+\log(M/\delta))$, 
then, with probability at least $1-\delta$, we have 
\begin{align*}
\left\|\sum_{m=1}^M\nu_m(\hat\theta_m-\theta^*_m)\right\|_\mathbf{H}^2\leq\left(\sum_{m=1}^M \frac{\nu_m^2}{n_m}\right)\Tilde{\mathcal O}\left(k(d+M)\log(1/\delta)^4\right).
\end{align*}
\end{lemma}

\begin{proof}
We assume that statements of Lemma \ref{lem:duelinghessian} and \ref{lem:duelingnoise} hold. 

\textbf{Step 1: Using the definition of $\hat \theta$.} The optimality of $\hat \theta_m$ states that 
\begin{align*}
\sum_{m=1}^M\sum_{i=1}^{n_m}\ell\left(x_{m,i}^\top \hat\theta_m, y_{m,i}\right)-\lambda\sum_{m=1}^Mn_m\|\hat\theta_m\|_2^2\geq 
\sum_{m=1}^M\sum_{i=1}^{n_m}\ell\left(x_{m,i}^\top \theta_m^*, y_{m,i}\right)-\lambda\sum_{m=1}^Mn_m\|\theta^*_m\|_2^2.
\end{align*}
Now apply Lemma \ref{lem:duelingappro} ($t^*=x_{m,i}^\top \theta_m^*, t=x_{m,i}^\top \hat\theta_m$) to all terms in the summation, we get
\begin{align}
\label{eq:form1}
\begin{split}
&\frac{1}{20L}\sum_{m=1}^M\sum_{i=1}^{n_m}\mu'(x_{m,i}^\top \theta_m^*)\left(x_{m,i}^\top \hat\theta_m-x_{m,i}^\top \theta_m^*\right)^2+\lambda\sum_{m=1}^Mn_m\|\hat\theta_m-\theta_m^*\|_2^2\\
\leq&
\sum_{m=1}^M\sum_{i=1}^{n_m}\left(y_{m,i}-\mu(x_{m,i}^\top \theta_m^*)\right)\left(x_{m,i}^\top \hat\theta_m-x_{m,i}^\top \theta_m^*\right)+2\lambda\sum_{m=1}^M n_m{\theta_m^*}^\top\left(\hat\theta_m-\theta_m^*\right).
\end{split}
\end{align}
Let $\Delta_m = \hat\theta_m-\theta^*_m$ and rewrite the above to matrix form to get
\begin{align}
\label{eq:form2}
\sum_{m=1}^M\left\|\Delta_m\right\|_{\mathbf{\hat H}_m}^2 \lesssim
\sum_{m=1}^M\eta_m^\top \mathbf X_m\Delta_m + 2\lambda\sum_{m=1}^M n_m\Delta_m^\top\theta^*_m. 
\end{align}

\textbf{Step 2: Bounding the right side.} Now we can bound two terms on the right side. 

For the first, we denote $E$ as the space spanned by vectors $\mathbf{H}^{\frac{1}{2}}\theta_1,\dots,\mathbf{H}^{\frac{1}{2}}\theta_M,\mathbf{H}^{\frac{1}{2}}\hat\theta_1,\dots,\mathbf{H}^{\frac{1}{2}}\hat\theta_M$, that satisfies $\text{rank}(E)\leq 2k$ due to Assumption \ref{smp:lowdim}. Let $\mathbf U$ be the $\text{rank}(E)\times d$ projection matrix, such that $x\mapsto \mathbf Ux$ is the projection from $\mathbb{R}^d$ to $E$. Using Cauchy-Schwarz inequality, we continue our bound as follows:  
\begin{align*}
\begin{split}
\sum_{m=1}^M\eta_m^\top X_m\Delta_m &\lesssim 
\sum_{m=1}^M\left\langle {\mathbf{\hat H}_m}^{-\frac{1}{2}}\mathbf X_m^\top \eta_m, \sqrt{n_m}\mathbf{H}^{\frac{1}{2}}\Delta_m\right\rangle =\sum_{m=1}^M \left\langle g_m, \sqrt{n_m}\mathbf{H}^{\frac{1}{2}}\Delta_m\right\rangle\\
&=\sum_{m=1}^M \left\langle \mathbf Ug_m, \sqrt{n_m}\mathbf{H}^{\frac{1}{2}}\Delta_m\right\rangle 
\leq \left(\sum_{m=1}^M\|\mathbf Ug_m\|_2^2\right)^\frac{1}{2}\left(\sum_{m=1}^Mn_m\|\Delta_m\|_\mathbf{H}^2\right)^\frac{1}{2}.
\end{split}
\end{align*}
The noise term above can be bounded by the concentration inequality and low-rank assumption, 
\begin{align*}
\sum_{m=1}^M\|\mathbf Ug_m\|_2^2
=\|\mathbf U\mathbf G\|_F^2\leq \text{rank}(E)\|\mathbf U\mathbf G\|^2 \leq 2k\|\mathbf G\|^2\leq \Tilde{\mathcal O}\left(k(d+M)\log(1/\delta)^4\right).
\end{align*}

For the second term, we have
\begin{align*}
\begin{split}
&\sum_{m=1}^M n_m\Delta_m^\top\theta^*_m  
\leq \left(\sum_{m=1}^M n_m\|\theta^*_m\|_{\mathbf{H}^{-1}}^2\right)^\frac{1}{2}\left(\sum_{m=1}^M n_m\|\Delta_m\|_{\mathbf{H}}^2\right)^\frac{1}{2}\\
\lesssim& \left(\sum_{m=1}^M n_mB_\theta^2\lambda^{-1}\right)^\frac{1}{2}\left(\sum_{m=1}^M n_m\|\Delta_m\|_\mathbf{H}^2\right)^\frac{1}{2}
\lesssim  \lambda^{-1}
\left(kM\right)^\frac{1}{2}\left(\sum_{m=1}^M n_m \|\Delta_m\|_\mathbf{H}^2\right)^\frac{1}{2}.
\end{split}
\end{align*}

\textbf{Step 3: Finishing the proof.} Given $\sum_{m=1}^M \|\Delta_m\|_{\mathbf{\hat H}_m}^2 \gtrsim\sum_{m=1}^Mn_m\|\Delta_m\|_\mathbf{H}^2$, we conclude that 
\begin{align}
\label{eq:duelingerror}
\sum_{m=1}^Mn_m\|\Delta_m\|_\mathbf{H}^2\leq \Tilde{\mathcal O}\left(k(d+M)\log(1/\delta)^4\right).
\end{align}

By applying Cauchy-Schwarz inequality, we get
\begin{align*}
\begin{split}
\left\|\sum_{m=1}^M\nu^m(\hat\theta_m-\theta^*_m)\right\|_\mathbf{H}^2&=\left\|\sum_{m=1}^M \nu_m\Delta_m\right\|_\mathbf{H}^2
\leq\left(\sum_{m=1}^M \frac{\nu_m^2}{n_m}\right)\left(\sum_{m=1}^M n_m\|\Delta_m\|_\mathbf{H}^2\right)\\
&\leq\left(\sum_{m=1}^M \frac{\nu_m^2}{n_m}\right)\Tilde{\mathcal O}\left(k(d+M)\log(1/\delta)^4\right).
\end{split}
\end{align*}
This completes the proof. 
\end{proof}

\subsection{Proof of Main Theorem}
\begin{lemma}
\label{thm:duelingpolicy}
Suppose Assumptions \ref{smp:lowdim}, \ref{smp:bound}, and \ref{smp:var} hold. Then for sufficiently small $\varepsilon > 0$, with probability at least $1-\delta$, Algorithm \ref{alg:dueling} outputs an $\varepsilon$-optimal policy when all parameters are chosen properly.
\end{lemma}

\begin{proof}
\textbf{Step 1: Accuracy of $\mathbf \Lambda$.} 
Because of \eqref{eq:duelingerror} in Lemma \ref{thm:duelinggeneral}, for sufficiently small $\varepsilon$, we have that the norm of $\hat\theta_m-\theta^*_m$ is also sufficiently small and thus $|x_{m,i}^\top\hat\theta_m-x_{m,i}^\top\theta^*_m|\leq 1, \forall m\in[M], i\in[n_m]$. This yields
\begin{align*}
\frac{1}{e}\mu'(x_{m,i}^\top\theta^*_m)\leq \mu'(x_{m,i}^\top\hat\theta_m)\leq e\mu'(x_{m,i}^\top\theta^*_m)
\end{align*}
Therefore, 
\begin{align*}
0.3\mathbf{\hat H_m}\preceq \sum_{i=1}^{n_m} \mu'(x_{m,i}^\top\hat\theta_m)x_{m,i}x_{m,i}^\top+n_m\lambda \mathbf I\preceq 3\mathbf{\hat H_m}.
\end{align*}
Combining results in Lemma \ref{lem:duelinghessian}, we have, with probability at least $1-\delta$, \begin{align*}
0.2C_1\mathbf{H} \preceq{\mathbf{\Lambda}}\preceq 4C_2\mathbf{H}.
\end{align*}

\textbf{Step 2: Bounding the gap.} We have, with probability at least $1-\delta$, 
\begin{align*}
\begin{split}
\|\hat\theta-\theta^*\|^2_{\mathbf{\Lambda}} \lesssim \|\hat\theta-\theta^*\|_\mathbf{H}^2&=\left\|\sum_{m=1}^M \nu_m(\hat\theta_m-\theta^*_m)\right\|_\mathbf{H}^2
\leq \Tilde{\mathcal O}\left(\frac{\varepsilon^2}{C^*}\right).
\end{split}
\end{align*}

With proper choices of parameters, $\|\hat\theta-\theta^*\|_{\mathbf{\Lambda}}\leq \alpha\varepsilon$ holds, namely $\theta^*\in \mathcal C$. Let 
\begin{align*}
\dot\theta = \argmin_{\theta\in\mathcal{C}}\left\{\mathbb E_{s\sim\rho}[\phi(s, \pi(s))^\top\theta]\right\}, \quad \Tilde\theta = \argmin_{\theta\in\mathcal{C}}\left\{\mathbb E_{s\sim\rho}[\phi(s, \pi^*(s))^\top \theta]\right\}. 
\end{align*}
Therefore, we continue our bound as follows:
\begin{align*}
\begin{split}
&\mathbb E_{s\sim\rho}\left[\phi(s, \pi^*(s))^\top\theta^* - \phi(s, \pi(s))^\top\theta^*\right]\\
=& \mathbb E_{s\sim\rho}\left[\phi(s, \pi^*(s))^\top\theta^* - \phi(s, \pi^*(s))^\top\Tilde\theta\right] 
+ \mathbb E_{s\sim\rho}\left[\phi(s, \pi^*(s))^\top\Tilde\theta - \phi(s, \pi(s))^\top\dot\theta\right]\\ 
+& \mathbb E_{s\sim\rho}\left[\phi(s, \pi(s))^\top\dot\theta - \phi(s, \pi(s))^\top\theta^*\right]\\
\leq& \left\|\mathbb E_{s\sim\rho}\left[\phi(s, \pi^*(s))\right]\right\|_{\mathbf{H}^{-1}}\|\theta^*-\Tilde\theta\|_{\mathbf{H}} + 0 + 0\\
\leq& \sqrt{C^*}2\alpha\varepsilon \lesssim \varepsilon. 
\end{split}
\end{align*}
This completes the proof. 
\end{proof}

\section{Proof of Theorem \ref{thm:unknown}} \label{appendix:pf:thm:unknown}

To get the precise sample complexity in the estimating $\nu^*$ step, we need the following term. It is similar to the diverse task assumption mentioned in \citet{du2020few}, \citet{tripuraneni2021provable} and \citet{wang2023improved}. We define a term $C_\theta$ and state it more explicitly. Intuitively, $C_\theta$ is a measure of how well the linear combinations of $\theta^*_m$ cover all directions in the space spanned by $\theta^*_m$. What we need is that any $\theta'$ with unit $\mathbf H-$norm ($\mathbf H$ is defined in \eqref{eq:activecover}) in the space can be represented by a linear combination of $\theta^*_m$ with small coefficients. The term $C_\theta$ only affects sample complexity in the step that estimates $\nu^*$, which is independent of $\varepsilon$. 

\begin{definition}
\label{asp:wellrepresent}
Let $C_\theta>0$ be the minimal number that makes the following true. 

For any $\theta'\in{\rm span}\{\theta^*_1,\dots,\theta^*_M\}$ with norm $\|\theta'\|_\mathbf{H}=1$, there exists $\alpha\in \mathbb R^M$ with norm $\|\alpha\|_2\leq C_\theta$, such that $\theta' = \sum_{m=1}^M \alpha_m\theta^*_m$.
\end{definition}

\begin{theorem}[Formal Version of Theorem \ref{thm:unknown}]
Let Assumptions \ref{smp:lowdim}, \ref{smp:bound}, and \ref{smp:var} hold. 
 Then for any sufficiently small $\varepsilon > 0$, Algorithm \ref{alg:unknown} outputs an $\varepsilon$-optimal policy with probability at least $1-\delta$ when
\begin{align}
\begin{split}
\label{eq:activesmaplesize1}
N^\pre_\sou &\gtrsim \max\left(\Tilde{\mathcal O}\left(\max(C_\theta^2\frac{R}{\|\nu^*\|_1^2}, \frac{B_x^2}{\sigma_{\rm min}(\mathbf{E})})Mk(d+M)\log(1/\delta)^4\right), \rho_x^4M(d+\log(M/\delta))\right),\\
n^\pre &\gtrsim \max\left(C_\theta^2(M/\|\nu^*\|_1^2)d\log(d/\delta)^2, \rho_x^4(d+\log(1/\delta))\right),\\
N_\sou&\gtrsim \max\left(\Tilde{\mathcal O}\left(C^*\|\nu^*\|_1^2k(d+M)\log(1/\delta)^4\varepsilon^{-2}\right),
\rho_x^4M(d+\log(M/\delta))\right), \\
n &\gtrsim \max\left(C^*k\log(d/\delta)^2\varepsilon^{-2}, \rho_x^4(d+\log(1/\delta))\right).\\
\end{split}
\end{align}
Here, $\sigma_{\rm min}(\mathbf{E})$ is the smallest singular value of $\mathbf{E}$. The proper penalties $\lambda^\pre_\sou, \lambda^\pre, \lambda_\sou, \lambda$ are chosen satisfying
\begin{align}
\label{eq:activelambda}
B_x^2/d\lesssim N^\pre_\sou\lambda^\pre_\sou/M, n^\pre \lambda^\pre, N_\sou \lambda_\sou/M, n\lambda\lesssim k/B_\theta^2,
\end{align}
and the penalty term for the Lasso Programming and parameter $\alpha$ of the confidence set are
\begin{align}
\label{eq:activebeta}
\beta =\mathcal O\left(\|\nu^*\|_1/C_\theta^2M\right), \quad \alpha=\mathcal O\left(1/\sqrt{C^*}\right).
\end{align}
\end{theorem}

\begin{remark}
\label{remark:activesamplepre}
We remark the following detailed analysis for terms in $N^\pre_\sou, n^\pre$. 
\begin{itemize}
\item The term $R/\|\nu^*\|_1^2\geq 1$ is a kind of distance between prior $\nu^0$ and the true value $\nu^*$. When $\nu^0$ is just $\nu^*$ multiplied by a scalar, then $R$ can be $1$.
\item The term $C_\theta$ in Definition \ref{asp:wellrepresent} only appears in the sample complexity of the first step, which is not related to $\varepsilon$, and thus not the main sample complexity term. This term is used in the proof of Lemma \ref{lem:activetilde} to bound the estimation error of $\hat\nu^\pre$. 
\item The term $\sigma_\text{min}(\mathbf E)$ is used to guarantee that $\mathbf \Lambda$ is a good approximation of $\mathbf H$ with high probability (Lemma \ref{lemma:activeLambda}). It is hard to estimate $\theta_m^*$ when $\sigma_\text{min}(\mathbf E)$ is small, and thus hard to estimate $\mathbf H$ using samples. 
\end{itemize}
\end{remark}

\subsection{The First Step}

\begin{lemma}[Bounding error of source tasks]
\label{lem:activesource1}
The term $C_\theta$ is defined in \eqref{eq:activectheta}. For both $\nu=\nu^*$ and $\nu=\hat \nu^\pre$, with probability at least $1-\delta/8$, we have 
\begin{align*}
\left\|\sum_{m=1}^M \nu_m (\hat \theta^\pre_m-\theta^*_m)\right\|_\mathbf{H} \leq \varepsilon_\theta. 
\end{align*}
\end{lemma}

\begin{proof}
Using results in Lemma \ref{thm:duelinggeneral} (the only different is the $\lambda$ in the definition of $\mathbf H$, this $\mathbf{H}-$norm is smaller than that of Lemma \ref{thm:duelinggeneral}), we get, with probability at least $1-\delta/8$,  
\begin{align*}
\left\|\sum_{m=1}^M \nu_m (\hat \theta^\pre_m-\theta^*_m)\right\|_\mathbf{H}^2 \leq \left(\sum_{m=1}^M\frac{\nu_m^2}{n^\pre_m}\right)\Tilde{\mathcal O}(k(d+M)\log(1/\delta)^4). 
\end{align*}
For both $\nu=\nu^*$ and $\nu=\hat \nu^\pre$, we have $\nu\in \Theta_\nu$. This completes the proof. 
\end{proof}

\begin{lemma}[Bounding error of target task]
\label{lem:activetarget1}
The term $C_\theta$ is defined in \eqref{eq:activectheta}. With probability at least $1-\delta/8$, we have 
\begin{align*}
\left\|\hat \theta^\pre -\theta^*\right\|_\mathbf{H} \leq \varepsilon_\theta.
\end{align*}
\end{lemma}

\begin{proof}
\textbf{Step 1: Using the definition of $\hat\theta^\pre$.} We have
\begin{align*}
\sum_{i=1}^{n^\pre}\ell((x_i^\pre)^\top\hat\theta^\pre , y_i^\pre) - n^\pre\lambda^\pre  \|\hat\theta^\pre\|_2^2 \leq \sum_{i=1}^{n^\pre}\ell((x_i^\pre)^\top\theta^*, y_i^\pre) - n^\pre\lambda^\pre  \|\hat\theta^*\|_2^2.
\end{align*}
Now apply Lemma \ref{lem:duelingappro} to all terms in the summation and rewrite the above to matrix form as in \eqref{eq:form1}, \eqref{eq:form2}, let $\Delta = \hat\theta^\pre-\theta^*$, we obtain
\begin{align*}
\begin{split}
\left\|\hat\theta^\pre-\theta^*\right\|_{\mathbf{\hat H}^\pre}^2 \lesssim 
(\eta^\pre)^\top \mathbf X^\pre \Delta + 2n^\pre\lambda^\pre  \Delta^\top\theta^*, 
\end{split}
\end{align*}
where
\begin{align*}
\mathbf{\hat H}^\pre = \sum_{i=1}^{n^\pre} \mu'({\theta^*}^\top x_i^\pre)x_i^\pre (x_i^\pre)^\top + n^\pre\lambda^\pre \mathbf I, \quad \eta_i^\pre = y_i^\pre-\mu\left((x_i^\pre)^\top\theta^*\right). 
\end{align*}

\textbf{Step 2: Bounding the right side.} The first term can be bounded by 
\begin{align*}
\begin{split}
(\eta^\pre)^\top \mathbf X^\pre \Delta &\leq \|(\mathbf{\hat H}^\pre)^{-\frac{1}{2}}(\mathbf X^\pre)^\top\eta^\pre\|_2\|\Delta\|_{\mathbf{\hat H}^\pre}\lesssim \sqrt{d}\log(d/\delta)\|\Delta\|_{\mathbf{\hat H}^\pre}. 
\end{split}
\end{align*}
The last inequality holds with probability at least $1-\delta/16$ using the bound of the noise of a single task, that is Lemma \ref{lem: duelingsingle}. 

The second term can be bounded by 
\begin{align*}
\begin{split}
n^\pre\lambda^\pre  \Delta^\top\theta^*
&\leq  n^\pre\lambda^\pre \|(\mathbf{\hat H}^\pre)^{-\frac{1}{2}}\theta^*\|_2\|\Delta\|_{\mathbf{\hat H}^\pre}
\leq \sqrt{n^\pre\lambda^\pre} \|\theta^*\|_2\|\Delta\|_{\mathbf{\hat H}^\pre}
\lesssim \sqrt{d} \|\Delta\|_{\mathbf{\hat H}^\pre}. 
\end{split}
\end{align*}

\textbf{Step 3: Finishing the proof.} Because of well concentration of the Hessian matrix, which is given by Lemma \ref{lem:duelinghessian}, we have, with probability at least $1-\delta/16$, $\|\Delta\|_{\mathbf{\hat H}^\pre}^2 \gtrsim n^\pre \|\Delta\|_{\mathbf{H}}^2$.
Combining all bounds above, we obtain that, with probability at least $1-\delta/8$,  
\begin{align*}
\left\|\Delta\right\|_\mathbf{H}^2 \lesssim \frac{1}{n^\pre}d\log(d/\delta)^2. 
\end{align*}
This completes the proof. 
\end{proof}

\begin{lemma}[Accuracy of $\mathbf{\Lambda}$]
\label{lemma:activeLambda}
Since 
\begin{align*}
N^\pre\geq\Tilde{\mathcal O}\left(\frac{B_x^2}{\sigma_{\rm min}(\mathbf{E})}Mk(d+M)\log(1/\delta)^4\right),
\end{align*}
with probability at least $1-\delta/16$, we have
\begin{align*}
0.2C_1\mathbf{H} \preceq{\mathbf{\Lambda}}\preceq 4C_2\mathbf{H}.
\end{align*}
\end{lemma}

\begin{proof}
Because of \eqref{eq:duelingerror} in Lemma \ref{thm:duelinggeneral}, we have with high probability, 
\begin{align*}
\|\hat\theta^\pre_m-\theta^*_m\|_\mathbf{H}^2 \leq \Tilde{\mathcal O}\left(\frac{M}{N^\pre}k(d+M)\log(1/\delta)^4\right) \leq \Tilde{\mathcal O}\left(\frac{\sigma_{\rm min}(\mathbf{E})}{B_x^2}\right).
\end{align*}
With proper parameters, we obtain
\begin{align*}
|(x_{m,i}^\pre)^\top\hat\theta^\pre_m-(x_{m,i}^\pre)^\top\theta^*_m|\leq B_x\|\hat\theta^\pre_m-\theta^*_m\|_2\leq \frac{B_x}{\sqrt{\sigma_{\rm min}(\mathbf{E})}}\|\hat\theta^\pre_m-\theta^*_m\|_\mathbf{H}\leq 1, 
\end{align*}
for all $m\in[M], i\in[n_m]$. Therefore, we can finish the proof using the same arguments in Step $1$ of the proof of Theorem \ref{thm:duelingpolicy}. 
\end{proof}

\begin{lemma}[Bounding sample complexity]
\label{lem:activesamplecom}
Let the statements of Lemmas \ref{lem:activesource1} and \ref{lem:activetarget1} hold, We have 
\begin{align*}
\|\hat\nu^\pre\|_1\lesssim \|\nu^*\|_1. 
\end{align*}
\end{lemma}

\begin{proof}
The Lasso Programming \eqref{eq:activelatent} and Lemma \ref{lemma:activeLambda} imply that 
\begin{align}
\label{eq:activelatent1}
\frac{1}{2}\left\|\sum_{m=1}^M\hat\nu^\pre_m\hat\theta_m^\pre-\hat\theta^\pre\right\|_\mathbf{H}^2 + \beta\|\hat\nu^\pre\|_1\lesssim \frac{1}{2}\left\|\sum_{m=1}^M\nu^*_m\hat\theta_m^\pre-\hat\theta^\pre\right\|_\mathbf{H}^2 + \beta\|\nu^*\|_1. 
\end{align}
We can bound the first term on the right by 
\begin{align}
\begin{split}
\label{eq:activelatent2}
\left\|\sum_{m=1}^M\nu^*_m\hat\theta_m^\pre-\hat\theta^\pre\right\|_\mathbf{H}^2
&\lesssim \left\|\sum_{m=1}^M\nu^*_m(\hat\theta_m^\pre-\theta^*_m)\right\|_\mathbf{H}^2 +  \left\|\theta^*-\hat\theta^\pre\right\|_\mathbf{H}^2
\lesssim \varepsilon_\theta^2 \lesssim \beta \|\nu^*\|_1. 
\end{split}
\end{align}
Combining \eqref{eq:activelatent1} and \eqref{eq:activelatent2}, we get the statement of the lemma. 
\end{proof}

\begin{lemma}[Well estimation $\hat \nu^\pre$]
\label{lem:activetilde}
Suppose the statements of Lemmas \ref{lem:activesource1} and \ref{lem:activetarget1} hold. There exists $\Tilde\nu\in \mathbb R^M$, such that 
\begin{align*}
\sum_{m=1}^M\Tilde\nu_m\theta_m^* = \theta^*, \quad \sum_{m=1}^M\frac{\Tilde\nu_m^2}{\hat\alpha_m}\lesssim \|\nu^*\|_1^2,
\end{align*}
where $\hat\alpha_m = |\hat\nu^\pre_m|/2\|\hat\nu^\pre\|_1+1/2M$. 
\end{lemma}

\begin{proof}
Combining \eqref{eq:activelatent1} and \eqref{eq:activelatent2}, we obtain that 
\begin{align*}
\left\|\sum_{m=1}^M\hat\nu^\pre_m\hat\theta_m^\pre-\hat\theta^\pre\right\|_\mathbf{H}^2 \lesssim \varepsilon_\theta^2. 
\end{align*}
Thus, we have bound $\left\|\sum\hat\nu^\pre_m\theta^*_m-\theta^*\right\|_\mathbf{H}\lesssim \varepsilon_\theta$ because 
\begin{align*}
\begin{split}
\left\|\sum_{m=1}^M\hat\nu^\pre_m\theta^*_m-\theta^*\right\|_\mathbf{H}&\leq \left\|\sum_{m=1}^M\hat\nu^\pre_m(\theta^*_m-\hat\theta^\pre_m)\right\|_\mathbf{H}
+\left\|\sum_{m=1}^M\hat\nu^\pre_m\hat\theta_m^\pre-\hat\theta^\pre\right\|_\mathbf{H}+\left\|\hat\theta^\pre-\theta^*\right\|_\mathbf{H}.
\end{split}
\end{align*}
Assumption \ref{asp:wellrepresent} states that, there exists $\alpha\in\mathbb R^M$, such that $\|\alpha\|_2\lesssim C_\theta\varepsilon_\theta$ and 
\begin{align*}
\sum_{m=1}^M \alpha_m\theta^*_m = \sum_{m=1}^M\hat\nu^\pre_m\theta^*_m-\theta^*.
\end{align*}
Let $\Tilde\nu = \hat\nu^\pre-\alpha$, we can check the second statement by 
\begin{align*}
\begin{split}
\sum_{m=1}^M\frac{\Tilde\nu_m^2}{\hat\alpha_m} &\lesssim \sum_{m=1}^M\frac{|\hat\nu^\pre_m|^2+\alpha_m^2}{\hat\alpha_m}
\lesssim \sum_{m=1}^M\frac{|\hat\nu^\pre_m|^2}{|\hat\nu^\pre_m|/\|\hat\nu^\pre\|_1}
+\sum_{m=1}^M\frac{\alpha_m^2}{1/M}
\leq \|\hat\nu^\pre\|_1^2 + M C_\theta^2\varepsilon_\theta^2
\lesssim \|\nu^*\|_1^2, 
\end{split}
\end{align*}
where the last inequality utilizes Lemma \ref{lem:activesamplecom}. 
\end{proof}

\subsection{The Second Step}

\begin{lemma}[Bounding error of source tasks]
\label{lem:activesource2}
The term $\Tilde{\theta}$ is defined in \eqref{eq:activetildetheta}. Assume that the statements of Lemmas \ref{lem:activesource1}, \ref{lem:activetarget1} and \ref{lem:activetilde} hold, then with probability at least $1-\delta/8$, 
\begin{align*}
\left\|\Tilde\theta-\theta^*\right\|_\mathbf{H}\lesssim \frac{\varepsilon}{\sqrt{C^*}}. 
\end{align*}
\end{lemma}

\begin{proof}
Using results in Lemma \ref{thm:duelinggeneral}, we obtain that, with probability at least $1-\delta/8$,
\begin{align*}
\left\|\Tilde\theta-\theta^*\right\|_\mathbf{H}^2\leq \left(\sum_{m=1}^M\frac{\Tilde\nu_m^2}{n_m}\right)\Tilde{\mathcal O}(k(d+M)\log(1/\delta)^4).
\end{align*}
Using results in Lemma \ref{lem:activetilde}, 
\begin{align*}
\sum_{m=1}^M\frac{\Tilde\nu_m^2}{n_m} = \frac{1}{N_\sou}\left(\sum_{m=1}^M\frac{\Tilde\nu_m^2}{\hat{\alpha}_m}\right)\lesssim \frac{\|\nu^*\|_1^2}{N_\sou}. 
\end{align*}
\end{proof}

\begin{lemma}[Approximation of log-likelihood]
\label{lem:activeappro}
Assume real numbers $|t-t^*|\leq C, y\in \{0,1\}$, where $C>1$, then
\begin{align*}
\ell(t, y)\geq \ell(t^*, y)+(y-\mu(t^*))(t-t^*)-\frac{e^C}{\kappa(t^*)}(t-t^*)^2.
\end{align*}
\end{lemma}

\begin{proof}
Fix $t^*$ and $y$, denote
\begin{align*}
f(t) = \ell(t, y)- \ell(t^*, y)-(y-\mu(t^*))(t-t^*)+\frac{e^C}{\kappa(t^*)}(t-t^*)^2.
\end{align*}
The derivative of $f$ can be expressed as
\begin{align*}
f'(t) = (y-\mu(t))-(y-\mu(t^*))+\frac{2e^C}{\kappa(t^*)}(t-t^*) = \int_{t^*}^t g(s) {\rm d s},
\end{align*}
where $g(s)=-\mu'(s)+2e^C/(\kappa(t^*))$. For $|\varepsilon|\leq C$, we have 
\begin{align*}
g(t^*+\varepsilon) 
= -\mu'(t^*+\varepsilon) +\frac{2e^C}{\kappa(t^*)}
\geq-e^C\mu'(t^*)+\frac{2e^C}{\kappa(t^*)}\geq 0.
\end{align*}
Therefore, $f'(t)\geq 0, t^*\leq t\leq t^*+C$ and $f'(t)\leq 0, t^*-C\leq t\leq t^*$. For any $|t-t^*|\leq C$, we have $f(t)\geq f(t^*)\geq 0$. This completes the proof. 
\end{proof}

\begin{lemma}[Bounding error of target task]
\label{thm:activemain}
For sufficiently small $\varepsilon > 0$, with probability at least $1-\delta$, we have 
\begin{align*}
\left\|\hat\theta-\theta^*\right\|_\mathbf{H}\lesssim\frac{\varepsilon}{\sqrt{C^*}}
\end{align*}
\end{lemma}

\begin{proof}
Here, we assume that the statements of Lemmas \ref{lem:activesource1}, \ref{lem:activetarget1}, and \ref{lem:activesource2} hold. 

\textbf{Step 1: Using definition of $\hat\theta$.} Because of Lemma \ref{lem:activesource2}, since $\varepsilon$ is sufficiently small, we have $\Tilde\theta\in\Theta$ for proper $\Theta$. The optimality of $\hat \theta$ states that
\begin{align*}
\sum_{i=1}^n\ell(x_i^\top\hat\theta,y_i)-\lambda n\|\hat\theta\|_2^2\geq \sum_{i=1}^n\ell(x_i^\top\Tilde\theta,y_i)-\lambda n\|\Tilde\theta\|_2^2.
\end{align*}
We set $t^*=x_i^\top\theta^*$ in Lemmas \ref{lem:duelingappro} and \ref{lem:activeappro}. Apply Lemma \ref{lem:duelingappro} to each term on the left, and since $|x_i^\top\theta^*-x_i^\top\Tilde\theta|\lesssim \varepsilon \|x_i\|_{\mathbf{H}^{-1}}/\sqrt{C^*}\lesssim 1$ (Lemma \ref{lem:activesource2} and $\varepsilon$ is sufficiently small), we can apply Lemma \ref{lem:activeappro} to each term on the right. It yields that there is an absolute constant $C>1$ such that
\begin{align*}
\begin{split}
&\frac{1}{20L}\sum_{i=1}^n\mu'(x_i^\top\theta^*)\left(x_i^\top\hat\theta-x_i^\top\theta^*\right)^2 + \lambda n\|\hat\theta-\theta^*\|_2^2 \\
\leq&C\sum_{i=1}^n\mu'(x_i^\top\theta)\left(x_i^\top\Tilde\theta-x_i^\top\theta^*\right)^2 + \sum_{i=1}^n\left(y_i-\mu(x_i^\top\theta^*)\right)(x_i^\top\hat\theta-x_i^\top\Tilde\theta)\\+&\lambda n \|\Tilde\theta - \theta^*\|_2^2 - 2\lambda n (\hat\theta-\Tilde\theta)^\top\theta^*.
\end{split}
\end{align*}
Rewrite the above inequality in matrix form as
\begin{align}
\label{eq:activemainineq}
\begin{split}
\left\|\hat\theta-\theta^*\right\|_{\mathbf{\hat H}}^2
\lesssim
C\left\|\Tilde\theta-\theta^*\right\|_{\mathbf{\hat H}}^2 
+ \eta^\top \mathbf X(\hat\theta-\Tilde\theta)
- 2\lambda n (\hat\theta-\Tilde\theta)^\top{\theta^*},
\end{split}
\end{align}
where 
\begin{align*}
\mathbf{\hat H} = \sum_{i=1}^n \mu'({\theta^*}^\top x_i)x_ix_i^\top + \lambda nI, \quad \eta_i = y_i-\mu(x_i^\top\theta^*). 
\end{align*}

Apply Lemma \ref{lem:duelinghessian} ($M=1$), because $\lambda\lesssim kM/(NB_\theta^2)$ is sufficiently small as $\varepsilon$ is sufficiently small, we have $\mathbf{E}\succeq \lambda \mathbf I$. Therefore, with probability $1-\delta/16$, we have 
\begin{align}
\label{eq:activemainhessian}
0.9C_1\mathbf{H}\preceq \frac{1}{n}\mathbf{\hat H} \preceq 2.2C_2\mathbf{H}.
\end{align}

\textbf{Step2: Bounding right side of \eqref{eq:activemainineq}.} The first term can be bounded by \eqref{eq:activemainhessian} and Lemma \ref{lem:activesource2}, 
\begin{align*}
\left\|\Tilde\theta-\theta^*\right\|_{\mathbf{\hat H}}^2 \lesssim n \left\|\Tilde\theta-\theta^*\right\|_{\mathbf{H}}^2 \lesssim n\varepsilon^2/C^*. 
\end{align*}

The third term can be bounded by Lemma \ref{lem:activesource2}, $\mathbf{E}\succeq \lambda \mathbf I$, and the bounded assumption of $\theta^*$, 
\begin{align*}
\lambda n (\hat\theta-\Tilde\theta)^\top\theta^*  \leq \lambda n \|\theta^*\|_{\mathbf{H}^{-1}}\|\hat\theta - \Tilde\theta\|_{\mathbf{H}}\lesssim  \sqrt{\lambda}nB_\theta\|\hat\theta - \Tilde\theta\|_\mathbf{H}\lesssim \sqrt{nk}\left(\|\hat\theta-\theta^*\|_\mathbf{H}+\varepsilon/\sqrt{C^*}\right). 
\end{align*}

Finally, we can bound the second term using low-rank property (Assumption \ref{smp:lowdim}) and the Matrix Bernstein's inequality. There exists a $k\times d$ projection matrix $\mathbf U$ such that $x\mapsto \mathbf Ux$ is the projection from $\mathbb R^d$ to ${\rm span}\{{\mathbf{\hat H}}^{\frac{1}{2}}\hat\theta_1,\dots,{\mathbf{\hat H}}^{\frac{1}{2}}\hat\theta_M\}$. Then
\begin{align*}
\begin{split}
\eta^\top \mathbf X(\hat\theta-\Tilde\theta)& = \langle {\mathbf{\hat H}}^{-\frac{1}{2}}\mathbf X^\top\eta, {\mathbf{\hat H}}^{\frac{1}{2}}(\hat\theta-\Tilde\theta)\rangle = \langle \mathbf U{\mathbf{\hat H}}^{-\frac{1}{2}}\mathbf X^\top\eta, {\mathbf{\hat H}}^{\frac{1}{2}}(\hat\theta-\Tilde\theta)\rangle\\
&\leq \left\|\mathbf U{\mathbf{\hat H}}^{-\frac{1}{2}}\mathbf X^\top\eta\right\|_2\|\hat\theta-\Tilde\theta\|_{\mathbf{\hat H}}
\lesssim \left\|\mathbf U{\mathbf{\hat H}}^{-\frac{1}{2}}\mathbf X^\top\eta\right\|_2\sqrt{n}\left(\|\hat\theta-\theta^*\|_\mathbf{H}+\varepsilon/\sqrt{C^*}\right). 
\end{split}
\end{align*}

\textbf{Step 3: Applying Matrix Bernstein's inequality.} Because $\{(x_i,y_i)\}_{i=1}^n$ are independent of $\mathbf U$, we can view $\mathbf U$ as fixed. Let $A_i=\eta_i\mathbf U{\mathbf{\hat H}}^{-\frac{1}{2}}x_i$ to be a $d\times 1$ matrix, we have
\begin{align*}
\left\|\mathbf U{\mathbf{\hat H}}^{-\frac{1}{2}}\mathbf X^\top\eta\right\|_2 = \left\|\sum_{i=1}^n A_i\right\|_2 = \left\|\sum_{i=1}^n A_i\right\|. 
\end{align*}

The matrix $A_i$ is bounded using \eqref{eq: gmbound}, 
\begin{align*}
\|A_i\|_2 \leq \left\| \eta_i {\mathbf{\hat H}}^{-\frac{1}{2}}x_i\right\|_2 \leq \mathcal O(\sqrt{d}).
\end{align*}
Also, let $B_i=\eta_i\mathbf{\hat H}^{-\frac{1}{2}}x_i$, it follows that
\begin{align*}
\sum_{i=1}^n \mathbb E B_i B_i^\top
=  {\mathbf{\hat H}}^{-\frac{1}{2}}\left(\sum_{i=1}^n (\mathbb E \eta_i^2)x_i x_i^\top \right){\mathbf{\hat H}}^{-\frac{1}{2}} \preceq \mathbf I. 
\end{align*}
Therefore, we can bound 
\begin{align*}
\begin{split}
\left\|\sum_{i=1}^n \mathbb E A_i^\top A_i\right\| &= {\rm tr}\left(\mathbb E\sum_{i=1}^n B_i^\top \mathbf U^\top \mathbf U B_i\right) =  {\rm tr}\left(\mathbf U^\top \left(\sum_{i=1}^n \mathbb E B_i B_i^\top \right)\mathbf U\right)\leq {\rm tr}\left(\mathbf U^\top \mathbf U \right) = k, \\
\left\|\sum_{i=1}^n \mathbb E A_i A_i^\top\right\|
&=  \left\|\mathbf U \left(\sum_{i=1}^n \mathbb E B_i B_i^\top \right)\mathbf U^\top \right\|\leq \left\|\mathbf U \mathbf U^\top \right\| \leq 1. 
\end{split}
\end{align*}
Apply Lemma \ref{lem:matrixbern}, we have, with probability at least $1-\delta/16$,  
\begin{align*}
\left\|\mathbf U{\mathbf{\hat H}}^{-\frac{1}{2}}\mathbf X^\top\eta\right\|_2 \lesssim \sqrt{k}\log(d/\delta). 
\end{align*}

\textbf{Step 4: Finishing the proof.} Summarizing all the bounds above, we have with probability at least $1-\delta/8$, 
\begin{align*}
n\|\hat\theta-\theta^*\|_{\mathbf{H}}^2 \lesssim n\varepsilon^2/C^* + \left(\sqrt{nk}\log(d/\delta)+ \sqrt{nk}\right)\left( \|\hat\theta- \theta^*\|_{\mathbf{H}}+\varepsilon/\sqrt{C^*}\right).
\end{align*}
We conclude that 
\begin{align*}
\|\hat\theta-\theta^*\|_{\mathbf{H}}\lesssim \frac{\varepsilon}{\sqrt{C^*}}. 
\end{align*}
This completes the proof. 
\end{proof}

\begin{lemma}
\label{thm:activepolicy}
Let Assumptions \ref{smp:lowdim}, \ref{smp:bound}, and \ref{smp:var} hold. Then for any sufficiently small $\varepsilon > 0$, Algorithm \ref{alg:unknown} outputs an $\varepsilon$-optimal policy with probability at least $1-\delta$ when all parameters are chosen properly. 
\end{lemma}

\begin{proof}
We first assume that 
\begin{align*}
\|\hat\theta-\theta^*\|_{\mathbf{H}}\lesssim \frac{\varepsilon}{\sqrt{C^*}}
\end{align*}
holds. Because of Lemma \ref{lemma:activeLambda}, we can assume
\begin{align*}
0.2C_1\mathbf{H}\preceq \mathbf\Lambda \preceq 4C_2\mathbf{H}. 
\end{align*}
The second step of the proof is the same as that of Theorem \ref{thm:duelingpolicy}.
We can prove the theorem. 
\end{proof}

\section{Important Lemmas}

\begin{lemma}[Covariance estimation]
\label{lem:varianceestimation}
Let $\mathbf A$ be an $n\times d$ matrix whose rows $A_i$ are independent, mean zero, sub-gaussian isotropic random vectors in $\mathbb{R}^d$ and $\left\Vert A_i\right\Vert_{\psi_2}\leq \sigma$, $\sigma\geq 1$. There is an absolute constant $C$ such that the following holds. For any $0<\delta<1$, if $n\geq C\sigma^4(d+log(1/\delta))$, then we have
\begin{align*}
0.9\mathbf I \preceq \frac{1}{n}\mathbf A^\top \mathbf A = \frac{1}{n}\sum_{i=1}^n A_i^\top A_i \preceq 1.1\mathbf I,
\end{align*}
with probability at least $1-\delta$. 
\end{lemma}

\begin{proof}
We adopt the result in Theorem 4.6.1 of \citet{vershynin2020high} and get 
\begin{align*}
\left\Vert \frac{1}{n}\mathbf A^\top  \mathbf A-\mathbf I\right\Vert\leq \sigma^2\max(\varepsilon, \varepsilon^2)\quad \text{where} \quad \varepsilon=c(\frac{\sqrt{d}+t}{\sqrt{n}}),
\end{align*}
with some constant $c>0$ and probability at least $1-2\exp(-t^2)$. If we take $t=\sqrt{\log(2/\delta)}$, then $\sqrt n\gtrsim c\sigma^2(\sqrt{d}+t)$. Thus, we have
\begin{align*}
0.9\mathbf I \preceq \frac{1}{n}\mathbf A^\top \mathbf A \preceq 1.1\mathbf I,
\end{align*}
with probability at least $1-\delta$. This completes the proof.
\end{proof}

\begin{lemma}[Matrix Bernstein's inequality for rectangular matrices]
\label{lem:matrixbern}
Let $\mathbf A_1, \dots, \mathbf A_N$ be independent, mean zero, $m\times n$ random matrices, such that $\|\mathbf A_i\|\leq K$ almost surely for all $i$. Then for $t\geq 0$, we have 
\begin{align*}
\mathbb P\left(\left\|\sum_{i=1}^N \mathbf A_i\right\|\geq t\right)\leq 2(m+n)\exp\left(-\frac{t^2/2}{\sigma^2+Kt/3}\right),
\end{align*}
where
\begin{align*}
\sigma^2=\max\left(\left\|\sum_{i=1}^N\mathbb E \mathbf A_i^\top \mathbf A_i\right\|, \left\|\sum_{i=1}^N\mathbb E \mathbf A_i \mathbf A_i^\top\right\|\right).
\end{align*}
\end{lemma}

\begin{proof}
This lemma is just Exercise 5.4.15 of \citet{vershynin2020high}. 
\end{proof}

\end{document}